\documentclass[preprint,12pt,3p]{elsarticle}
\usepackage{url}
\usepackage{amssymb}
\usepackage{float}
\usepackage{caption}
\usepackage{subcaption}
\usepackage{dsfont}
\usepackage{graphicx}
\usepackage{hyperref}
\newif\ifblind
\usepackage{amsmath}
\usepackage{amsthm}
\usepackage{algorithm}
\usepackage{algpseudocode}
\usepackage{mathtools}
\usepackage{textcomp}
\usepackage{accents}
\usepackage{xcolor}

\usepackage{gensymb}
\usepackage{comment}
\newtheorem{theorem}{Theorem}
\newtheorem{lemma}{Lemma}
\newtheorem{corollary}{Corollary}

\usepackage{comment}
\usepackage{csquotes}

\MakeOuterQuote{"}

\usepackage[justification=centering]{caption}

\journal{Aerospace Science and Technology}

\begin{document}

\begin{frontmatter}

\title{Global Sensitivity Analysis for Engineering Design Based on Individual Conditional Expectations}

% \title{A Global Sensitivity Metric based on Individual Conditional Expectations: Applications to Engineering Design Exploration}
\author[itb]{Pramudita Satria Palar\corref{cor1}}
\ead{pramsp@itb.ac.id}

\author[irit]{Paul Saves}

\author[sju]{Rommel G. Regis}

\author[kyushu]{Koji Shimoyama}

\author[tohoku]{Shigeru Obayashi}

\author[irit]{Nicolas Verstaevel}

\author[isae]{Joseph Morlier}

\cortext[cor1]{Corresponding author}

\address[itb]{Faculty of Mechanical and Aerospace Engineering, Institut Teknologi Bandung, Bandung, Indonesia}
\address[irit]{
%IRIT and UT Capitole, Universit\'e de Toulouse, CNRS, Toulouse INP, UPS, UT2J, UT Capitole, Toulouse, France
Université Toulouse Capitole, IRIT, Toulouse, France }
\address[sju]{Department of Mathematics, Saint Joseph's University, Philadelphia, PA, USA}
\address[kyushu]{Department of Mechanical Engineering, Kyushu University, Fukuoka, Japan}
\address[tohoku]{Institute of Fluid Science, Tohoku University, Sendai, Japan}
\address[isae]{ISAE-SUPAERO, Université de Toulouse, Toulouse, France}

\begin{abstract}
Explainable machine learning techniques have gained increasing attention in engineering applications, especially in aerospace design and analysis, where understanding how input variables influence \textcolor{black}{predictive} models is essential. Partial Dependence Plots (PDPs) are widely used for interpreting black-box models by showing the average effect of an input variable on the prediction. However, their global sensitivity metric can be misleading when strong interactions are present, as averaging tends to obscure interaction effects. To address this limitation, we propose a global sensitivity metric based on Individual Conditional Expectation (ICE) curves. The method computes the expected feature importance across ICE curves, along with their standard deviation, to more effectively capture the influence of interactions. \textcolor{black}{The proposed metrics are model-agnostic and can be applied to any predictive model, including but not limited to surrogate models.} Furthermore, we provide a mathematical proof demonstrating that the PDP-based sensitivity is a lower bound of the proposed ICE-based metric under additive and multiplicative separability. In addition, we introduce an ICE-based correlation value to quantify how interactions modify the relationship between inputs and the output. Comparative evaluations were performed on three cases: a 5-variable analytical function, a 5-variable wind-turbine fatigue problem, and a 9-variable airfoil aerodynamics case, where ICE-based sensitivity was benchmarked against PDP, SHapley Additive exPlanations (SHAP), and Sobol' indices. The results show that ICE-based feature importance provides richer insights than the traditional PDP-based approach, while visual interpretations from PDP, ICE, and SHAP complement one another by offering multiple perspectives.
\end{abstract}


\end{frontmatter}

%\tableofcontents

% As engineering problems grow more complex, especially in aerospace, precise simulation data is important. 

\section{Introduction}
High-fidelity simulations are essential in modern engineering design. \textcolor{black}{Therefore, aerospace engineers} often rely on computer simulations to assess and optimize performance, enabling accurate prediction of behavior under varied conditions and supporting informed, performance- and safety-driven decisions~\cite{saves2025surrogate}. Yet high-fidelity simulations remain challenging for design exploration because engineers can only afford to run a limited number of computationally intensive simulations. To address this, a surrogate model that estimates the relationship between input or design variables with the {Q}uantity of {I}nterest (QoI) plays an important role in accomplishing this task. The surrogate can be used \textcolor{black}{for} various purposes such as optimization~\cite{bhosekar2018advances} and uncertainty analysis~\cite{cheng2020surrogate}. Various types of surrogate models exist, among popular ones in engineering, including polynomial regression (including Polynomial Chaos Expansion (PCE)~\cite{ghanem1999stochastic}), Kriging/Gaussian process regression~\cite{rasmussen2003gaussian, saves2023mixed}, and deep neural network~\cite{lecun2015deep} to reduce the computational cost of inference.

Sensitivity analysis, particularly {G}lobal {S}ensitivity {A}nalysis (GSA), supports engineering endeavours by identifying which variables are most and least influential~\cite{razavi2021future}. In the context of engineering design exploration, this information is important for setting design priorities, such as determining which variables should be prioritized. For instance, in the field of aerodynamics, GSA aids in analyzing the impact of geometric or kinematic parameters on the aerodynamic performance of a multi-element airfoil~\cite{pm2023sensitivity}, a hovering flapping wing~\cite{lang2023sensitivity}, and a high-pressure capturing wing~\cite{xiaozhe2024surrogate}.  Common techniques in engineering include variance-based decomposition methods, such as {Sobol'} indices~\cite{Sobol2001global}, which measure the individual effects of input variables and their interactions on the {QoI}. Although variance-based decomposition is valuable, it only quantifies how much each input contributes to output variance but does not directly reveal the input–output functional form, such as local nonlinearity or the detailed structure of interactions. 

% This is why we argue that visual tools such as PDPs and ICE curves provide complementary information to variance decomposition.

% The simplest approach is to take a slice of the black-box function, but this method offers only a narrow perspective on the relationship between an input variable and the QoI, which can overlook essential nonlinearities, interactions, and the broader context. 

Besides GSA, sometimes it is also of interest to understand the behavior of the black-box function when the input of interest is changed. In artificial intelligence and \textcolor{black}{Machine Learning (ML), explainability serves several purposes}, including building trust in the model and enhancing transparency~\cite{belle2021principles, roscher2020explainable}. Even in engineering, transparency is also vital, as it allows engineers and researchers to verify that the ML model is consistent with established domain knowledge. In this paper, we aim to utilize explainability methods for knowledge discovery in engineering design, viewing the ML model not just as a predictive tool, but also as a source of valuable insights.  Recently, there has been growing interest in unifying the frameworks of sensitivity analysis and explainable ML~\cite{scholbeck2023position}.

%Explainability methods are more useful than simply slicing a function because they offer a more comprehensive understanding of how input variables influence the QoI across the entire input space, capturing complex nonlinear relationships and interactions that a single slice might miss.

Techniques from explainable ML and graphical statistics such as {P}artial {D}ependence {P}lot (PDP)~\cite{friedman2001greedy}, {I}ndividual {C}onditional {E}xpectation (ICE)~\cite{goldstein2015peeking}, and {SH}apley {A}dditive ex{P}lanations (SHAP)~\cite{lundberg2017unified} can be leveraged to understand the {behavior} of the black-box function. These methods provide insights into how input variables influence the predictive model and interact with each other. Such information is invaluable in engineering and science, helping engineers and scientists extract meaningful insights from the constructed model. The applications of explainable ML in engineering are now gaining significant attention, with examples not only in aerospace and mechanical engineering~\cite{shukla2020opportunities, shen2022automatic, wen2025explainable, kumar2024multisource}, but also in other fields such as material science~\cite{pilania2021machine}.

% While this paper does not emphasize computational time, one significant drawback is its computational complexity, particularly with high-dimensional data or complex models, where calculating exact SHAP values can become impractical. 

SHAP is arguably one of the most widely used explainability methods. However, despite its widespread adoption, SHAP has certain limitations. While SHAP offers robust insights, the specialized nature of its theoretical foundation may present a learning curve for practitioners not familiar with game theory concepts~\cite{kumar2020problems}. Although SHAP decompositions provide faithful local attributions, aggregating those local explanations into a single global summary can conceal heterogeneity across instances~\cite{herbinger2024decomposing}. 
In this context, PDP and ICE offer more intuitive explanations of an ML model due to their simplicity and straightforward depiction of relationships compared to SHAP. They function by altering feature values and observing the corresponding changes in predictions, a concept closely aligned with causal reasoning and easier to understand. Therefore, SHAP is best used together with PDP/ICE or conditional analyses when a clear global picture of feature interactions is required~\cite{mahmood2025machine}.

%The idea of attributing contributions to individual features based on how they would ``cooperate" to achieve a particular prediction is not immediately intuitive. 
This paper examines how PDP and ICE curves support engineering design exploration with surrogate models, especially for visualization and global sensitivity analysis. PDPs provide a convenient global sensitivity measure for input importance~\cite{greenwell2018simple}, but their averaging can mask interactions and lead to misleading conclusions if used alone. To address this, we propose to leverage ICE to inspect the individual ICE curves behind the PDP. Even if ICE is still restricted to independent features, it reveals heterogeneous effects and interactions hidden by averaging. We propose a new global sensitivity metric by computing and averaging the standard deviation of each ICE curve, thereby capturing the influence of variable interactions. Furthermore, the standard deviation of importance across all ICE curves is used to quantify the overall impact of interactions, while the ICE correlation value evaluates how such interactions modify the input–output relationship and its nonlinearity. The proposed approach is validated through three case studies: a 5-variable analytical function, a 5-variable wind turbine fatigue problem, and a 9-variable airfoil aerodynamics case.

\section{Motivations}
\label{sec:sota}
The widespread use of high-fidelity simulations in aerospace means that engineers must do more than predict the QoI: careful interpretation of black-box surrogate models is now required. This section briefly surveys the convergence of GSA and \textcolor{black}{Explainable Artificial Intelligence (XAI)}, and exposes that engineering requires sensitivity analysis that goes beyond classical variance decomposition.

\subsection{Global Sensitivity Analysis in Engineering Design}

% While robust, Sobol' indices face three main critical limitations in modern engineering workflows: Dependent inputs, computational cost, and scalar reduction.

% In aerospace design, it is also important to distinguish between aleatoric uncertainty, the stochastic variability inherent to the model, and epistemic uncertainty, the lack of knowledge that can be reduced~\cite{der2009aleatory}.

% Aerospace design variables are frequently coupled and therefore dependent, for example, a wing thickness constrained by a chord length.  Standard Sobol' indices assume independence, and ignoring dependencies leads to erroneous attributions. This has driven the adoption of Shapley effects~\cite{iooss2019shapley}, which allocate variance fairly even under correlation.

In high-dimensional engineering problems, GSA is essential for knowledge discovery and variable screening; the dominant framework is variance-based sensitivity methods that are widely considered the “gold standard” in various fields~\cite{Sobol2001global,sudret2008global}. For high-dimensional problems, estimating variance contributions can be prohibitive in terms of computational time, especially for dummy Monte Carlo implementation~\cite{Demange-Chryst}. This has led to the popularity of Derivative-based Global Sensitivity Measures (DGSM)~\cite{kucherenko2016derivative}, which use gradients to bound sensitivity indices, and Leave-one-covariate-out  metrics~\cite{lei2018distribution}, which focus on predictive power rather than variance. Still, many different techniques may alleviate the problem of sensitivity indices computation in high dimensions~\cite {li2024efficient,spagnol2020kernel}, but with a small number of data points to learn on in high dimensions, these indices remain difficult to estimate~\cite{becker2018sensitivity}.

Concerning the computational cost, it is worth noting that many authors have proposed leveraging surrogate models, such as PCE, to perform GSA for complex systems in a tractable and time-efficient manner~\cite{sudret2008global}. Even with many advances, sensitivity indices only provide a single scalar measure of influence, indicating that a variable is important but failing to describe the mechanism of influence. For instance, if the response is monotonically increasing, quadratic, or exhibits a threshold, an index will not reveal how the output behaves across the input's range. Understanding this detailed shape of the output response is a prerequisite for making informed design trade-offs~\cite{borgonovo2025direction}.

\subsection{The Rise of Explainable AI for Engineering}

XAI can complement scalar GSA by exposing the functional and instance-level behavior of complex surrogates used in aerospace design. In practice, XAI delivers two complementary views: \emph{local} explanations (why a specific prediction was made \textcolor{black}{for a specific input}) and \emph{global} explanations (how the model behaves across the input space). In engineering workflows, local explanations support root-cause analysis and failure investigation, while global views guide design exploration and trade-offs. The most used technique in such a context is SHAP~\cite{lundberg2017unified} that provides principled \emph{local} attributions based on cooperative game theory. 

% Importantly, SHAP (local) is distinct from Shapley effects~\cite{song2016Shapley} (a global GSA measure): aggregating per-instance SHAP values (\textit{e.g.}, mean absolute SHAP) yields a heuristic global ranking but can mask heterogeneity across the operational design space.

% While PDP can reveal the shape of a feature’s effect (\textit{e.g.}, linear, nonlinear), it does not directly equate to first-order variance decomposition (such as Sobol' main effects), because it averages across both main and interaction contributions rather than isolating just the variance attributed to that feature~\cite{molnar2019}.

On the other hand, PDP~\cite{friedman2001greedy} estimates a variable’s marginal profile by averaging the model output over other inputs.  PDPs are intuitive and widely used to visualise input-output trends, but their averaging hides heterogeneous interactions: when an input has opposite effects in different regions, PDPs can cancel these effects and produce a misleadingly flat profile, a serious hazard in safety-critical cases. On the other hand, ICE curves~\cite{goldstein2015peeking} plot per-instance response trajectories, revealing heterogeneity that PDP averages obscure. ICE therefore offers richer diagnostic power for engineering problems, in which understanding is critical. However, their utility in automated engineering design pipelines is hindered by the absence of standardised quantitative metrics or scalar summaries necessary for efficient variable ranking. A key research priority for broader engineering adoption is the development of ICE-based scalar metrics, which is the topic of the present paper. These XAI tools, when used with awareness of their assumptions and computational costs, form a complementary toolkit for interpretable aerospace design. %Given the Rashomon Effect, relying on a single surrogate or sensitivity analysis method creates a false sense of certainty, as it obscures the conflicting yet equally plausible variable importance rankings inherent in the model space~\cite{donnelly2023rashomon}.

% To address the limitations of the "cancellation effect" in PDPs, Goldstein et al.~\cite{goldstein2015peeking} introduced ICE plots. Instead of plotting the average, ICE plots display the functional relationship for every instance in the dataset, resulting in a set of curves that reveal the full dispersion of model behavior. 

% this work bridges the gap between the rigorous quantification of variance-based GSA and the visual interpretability of ICE.

\subsection{Positioning of the Present Work}
There is currently a lack of formalised metrics that translate the visual richness of ICE curves into scalar sensitivity indices suitable for automated ranking in engineering design. To that end, this work bridges the gap between GSA and the visual interpretability of ICE.  We propose novel \textit{ICE-based sensitivity metrics} that are based on the dispersion of ICE curves. By focusing on conditional expectations, we derive metrics that:
\begin{enumerate}
    \item Capture interaction effects that are mathematically cancelled out in standard PDP-based metrics.
    \item Provides a theoretical lower bound to the proposed ICE-based sensitivity metric.
    \item Offers a "correlation" metric to identify whether interactions purely scale the output or fundamentally alter the functional relationship.
\end{enumerate}
This approach allows engineers to retain the intuitive visualization of response curves while ensuring that critical interactions are not masked by averaging artefacts.
\textcolor{black}{This paper introduces sensitivity analysis tools derived from local conditional behavior that bridge the computational gap between local and global sensitivity analysis and provide a quantitative complement to graphical methods. Also, the methods proposed in this paper apply to any response function, including black-box models and surrogate representations. For computationally expensive models, such as those arising from computational fluid dynamics simulations, the proposed GSA metrics can be evaluated using a surrogate model, provided that the surrogate achieves sufficient predictive accuracy to ensure reliable sensitivity estimates and knowledge extraction.}

\section{\textcolor{black}{Surrogate Modeling for Global Sensitivity and Explainability Analysis}}

\subsection{Surrogate model}
\label{sec:surrogatemodel}

Let $\boldsymbol{x}=(x_{1},x_{2},\ldots,x_{m})^{T}$ be the vector of the input of interest, where $m$ is the number of input variables. Assuming independence between input variables, let us define the domain of the input variables as $\boldsymbol{\Omega}=\prod_{j=1}^{m}\Omega_{j}$, where $\Omega_{j}$ is the domain for the $j$-th input variable. Let us also define the joint density function $\boldsymbol{\pi}(\boldsymbol{x}) = \prod_{j=1}^{m} \pi_{x_{j}}(x_{j})$, where $\pi_{x_{j}}$ represents the marginal density of $x_j$. Also, let $y=f(\boldsymbol{x})$ be a black-box function that relates $\boldsymbol{x}$ to the {QoI} $y$. A surrogate model \(\hat{f}(\boldsymbol{x})\) essentially approximates the relationship between $y$ and $\boldsymbol{x}$, based on the given data set $\mathcal{D}=\{(\boldsymbol{x}^{(i)},y^{(i)})\}_{i=1}^{n}$, where $n$ is the size of the data set. In this paper, we use the data-driven non-intrusive PCE as our surrogate model of choice, although other models could also be applicable. 

%The surrogate model, \(\hat{f}(\boldsymbol{x})\), is typically constructed to approximate the {behavior} of the true function \(f(\boldsymbol{x})\) while significantly reducing computational costs associated with evaluating the original black-box function. This is especially valuable in scenarios where \(f(\boldsymbol{x})\) is expensive to evaluate, such as in high-fidelity simulations or complex physical models. % To build the surrogate model, various techniques can be employed, including polynomial regression, Gaussian processes, neural networks, and support vector machines, among others.

% PCE is selected due to its good accuracy in approximating the specific problems addressed in this work. Additionally, our primary focus is on dissecting the black-box function through explainability methods, aiming to gain deeper insights into the underlying relationships within the data. 

While more detailed explanations are given in~\cite{blatman2011adaptive}, the PCE model is briefly outlined below. The expansion reads as
\begin{equation}
f(\boldsymbol{x}) \approx \hat{f}(\boldsymbol{x}) = \sum_{\boldsymbol{\gamma} \in \mathbb{A}_{p}} \beta_{\boldsymbol{\gamma}} \boldsymbol{\Psi}_{\boldsymbol{\gamma}}(\boldsymbol{x}),
\end{equation}
where the $\boldsymbol{\Psi}_{\boldsymbol{\gamma}}$'s are multidimensional polynomials that are orthonormal with respect to the joint density \textcolor{black}{$\boldsymbol{\pi}(\boldsymbol{x})$}, and $\boldsymbol{\gamma}=\{\gamma_{1},\ldots,\gamma_{m}\}$ with $\gamma_{i}\geq 0$ for each $i$, denotes the individual indices which are members of $\mathbb{A}_{p}$. The coefficients $\boldsymbol{\beta}=\{\beta_{\gamma}\}$ corresponds to the terms of the PCE coefficients, where $P$ is the cardinality of $\mathbb{A}_{p}$. The index set $\mathbb{A}_{p}$ is constructed using a total-order truncation~\cite{blatman2011adaptive}:
\begin{equation}
\mathbb{A}_{p} \equiv \{\boldsymbol{\gamma} \in \mathbb{N}^{m}: \|\boldsymbol{\gamma}\|_{1} \leq p\},
\end{equation}
where $p$ denotes the polynomial order and
\begin{equation}
\|\boldsymbol{\gamma}\|_{1} = \sum_{i=1}^{m} \gamma_i.
\end{equation}

A PCE model is constructed by solving $\boldsymbol{F}\boldsymbol{\beta}=\boldsymbol{y}$, where $F_{ij} = \Psi_{j}(\boldsymbol{x}^{(i)})$ forms the $n \times P$ design matrix, with each column representing a polynomial term evaluated at the data points. Here, $\boldsymbol{y}=(y^{(1)},\ldots,y^{(n)}) = (f(\boldsymbol{x}^{(1)}),\ldots,f(\boldsymbol{x}^{(n)}))$ represents the responses. In this paper, this system of equations is solved using the least-angle regression %(LARS) 
algorithm. The optimal order $p$ is sought by exploring various \(\mathbb{A}_{p}\) sets, ranging from \(\mathbb{A}_{p=1}\) to \(\mathbb{A}_{p=p_{max}}\), where \(p_{max}\) represents the highest polynomial order considered.

\textcolor{black}{The method and the framework proposed in this paper, explained in detail in Sect.~\ref{sec:theory_ice},  are surrogate-agnostic and do not rely on any specific properties of PCE. PCE is adopted here due to its spectral and additive structure, which facilitates the theoretical derivation of the inequality in Eq.~(\ref{eq:ICE_ineq}) and enables analytical tractability, as detailed in Appendix~A. In addition, for the engineering demonstrations considered in this paper, the PCE surrogate achieves excellent predictive accuracy. }In this study, we employ the data-driven PCE module from UQLab~\cite{marelli2014uqlab}, whereby the orthogonal polynomials for the engineering problems are constructed directly from the available data by inferring the probability distributions first.

\subsection{Partial Dependence Plots and Individual Conditional Expectations}
% According to Zhao and Hastie~\cite{zhao2021causal}, effective explainability requires an accurate surrogate model and solid domain knowledge. PDP and ICE are emphasized here for their intuitive role in engineering design exploration.

% As highlighted by Zhao and Hastie~\cite{zhao2021causal}, a key aspect of explainability analysis is having a well-performing surrogate model $\hat{f}(\boldsymbol{x})$ that closely approximates the true function. Moreover, effective causal interpretation necessitates strong domain knowledge. PDP and ICE plots are particularly valuable tools in this context, aiding in the exploration and understanding. Consequently, this paper focuses on PDP and ICE due to their intuitive nature and their usefulness in supporting engineering design exploration.

%\subsubsection{Formulation}
Let us define $[1:m]=\{1,2,\ldots,m\}$. Next, let $S \subset [1:m]$ represent a subset of these indices, and let $C=[1:m] \setminus S$ denote the complementary set of indices not included in $S$. Also, let $\boldsymbol{x}_{S}=(x_{j}|j \in S)$ and $\boldsymbol{x}_{C}=(x_{j}|j \in C)$ represent the vectors of inputs corresponding to $S$ and $C$, respectively. Specifically, we want to build the partial dependence function of the output of the variables indexed by $S$. Let $\Omega_{S}$ be the subdomain for the variables indexed by $S$, {\textit{i.e.}}, $\boldsymbol{\Omega}_{S} = \prod_{j \in S} \Omega_{j}$. Similarly, $\boldsymbol{\Omega}_{C}=\prod_{j \in C}  \Omega_{j}$. Finally, we define the joint probability density functions for these subsets as $\boldsymbol{\pi}_{\boldsymbol{x}_{S}} = \prod_{j \in S} \pi_{x_j}(x_j)$ and $\boldsymbol{\pi}_{\boldsymbol{x}_{C}} = \prod_{j \in C} \pi_{x_j}(x_j)$.

The general formulation of the partial dependence function for $\boldsymbol{x}_{S}$ is written as:
\begin{equation}
\label{eq:PDP}
\hat{f}_{pdp}(\boldsymbol{x}_{S})  = \mathbb{E}_{\boldsymbol{x}_{C}}\left[\hat{f}\left(\boldsymbol{x}_{S}, \boldsymbol{x}_{C}\right)\right] = \int_{\boldsymbol{\Omega}_{C}} \hat{f}(\boldsymbol{x}_{S}, \boldsymbol{x}_{C}) \boldsymbol{\pi}_{\boldsymbol{x}_{C}}(\boldsymbol{x}_{C}) d\boldsymbol{x}_{C},
\end{equation}
which represents the marginal expectation of $\hat{f}(\boldsymbol{x})$ over $\boldsymbol{x}_{C}$. For example, if $S=\{1\}$, then we have an interest in building the partial dependence function for the first variable, that is $\hat{f}_{pdp}(x_{1})$. Similarly, if $S=\{1,2\}$, we have an interest in building the joint partial dependence function for $x_{1}$ and $x_{2}$, that is $\hat{f}_{pdp}(x_{1},x_{2})$. PDP essentially shows how a specific set of input variables ($\boldsymbol{x}_{S}$) affects the output of a model, by averaging out the effects of all other variables ($\boldsymbol{x}_{C}$). The partial dependence function can be conveniently calculated using random sampling as follows:
\begin{equation}
    \label{eq:PDP_MC}
    \hat{f}_{pdp}(\boldsymbol{x}_{S}) \approx \frac{1}{N} \sum_{i=1}^{N}\hat{f}(\boldsymbol{x}_{S}, \boldsymbol{x}_{C}^{(i)}),
\end{equation}
where $N$ is the size of the random sampling set and $\boldsymbol{x}_{C}^{(i)}$ is the $i$-th realization of samples drawn from $\boldsymbol{\pi}_{\boldsymbol{x}_{C}}(\boldsymbol{x}_{C})$.

Because PDP averages over other features, it can hide interaction effects; For example, showing a flat line when equal portions of data have positive and negative correlations with the output. This issue is resolved by displaying the {ICE} curves together with the main PDP~\cite{goldstein2015peeking}. A single ICE curve essentially represents a single  $\hat{f}(\boldsymbol{x}_{S}, \boldsymbol{x}_{C}^{(i)})$ (see Eq.~(\ref{eq:PDP_MC})). That is, a single ICE plot illustrates the effect of changing the input of interest while keeping the other variables constant.
To model interaction is to capture non-additivity and more complex effects. Visually, ICE diagnostics reveal the {structure} of interactions: heterogeneity in curve slopes indicates that the effect of \textcolor{black}{$\boldsymbol{x_S}$} depends on $\boldsymbol{x_{C}}$. 
% Conversely, variance-based indices quantify the {magnitude} of these interactions through variance decomposition. These approaches are complementary: ICE visualizes the form of heterogeneity, while scalar indices report its global weight. % Crucially, under feature dependence, one must explicitly distinguish between marginal and conditional computation to ensure a valid interpretation.

% Plotting multiple ICE curves allows us to visualize how the trend shifts as the input of interest interacts with other variables in influencing the QoI. 

%  In practice, plotting the ICE curve is only feasible for $|S|=1$ and $|S|=2$. 

Anchoring the PDP and ICE curves at a specific point can better visualize the interactions. In depicting PDP and ICE curves, we anchor the curves at the centre or mean of the input space. The anchored partial dependence function ({\textit{i.e.}}, $ \hat{f}_{pdp}^{\text{anc}}(\boldsymbol{x}_{S}$) is reads as
\begin{equation}
    \hat{f}_{pdp}^{\text{anc}}(\boldsymbol{x}_{S}) = \hat{f}_{pdp}(\boldsymbol{x}_{S})- \hat{f}_{pdp}(\boldsymbol{x}_{S}^{a}),
\end{equation}
where $\boldsymbol{x}_{S}^{a} = \mathbb{E}(\boldsymbol{x}_{S})$ is the anchor point. \textcolor{black}{The second term on the right-hand side corresponds to the value of the partial dependence function evaluated at the anchor point $\boldsymbol{x}_S^a$. Subtracting this constant value shifts the entire PDP so that it is zero at the anchor point, without altering its shape.} The ICE curves are also anchored using the same principle. 
In the discussions and results that follow, we always depict the PDP and ICE curves anchored through this procedure.
\textcolor{black}{
Both PDP and ICE may be misleading when dealing with correlated inputs. 
This work assumes that all inputs are independent, which is reasonable for design exploration or parameteric studies~\cite{kruskal1988miracles}.}

\section{Derivation of GSA indices from PDP and ICE}
\label{sec:theory_ice}

A key contribution of this work is to bridge the gap between the visual insight of ICE curves and formal global-sensitivity measures grounded in distributional distances. Below, we make that connection precise, discuss the assumptions, and highlight how our ICE-based metric can serve as a practical proxy for more general sensitivity measures.

\paragraph{Global sensitivity analysis}
PDP can provide a global sensitivity derived from the standard deviation of its partial dependence function, {\textit{i.e.}}, $I_{pdp}$. To make the notation simpler, let us focus on quantifying the importance of a single variable $x_{j}$. Let us first define the variance of the partial dependence function of $x_{j}$ as 

\begin{equation}
\text{Var}_{x_{j}}\big[\hat{f}_{pdp}(x_{j})\big] = \mathbb{E}_{x_{j}}\left[\left(\hat{f}_{pdp}(x_{j}) - \mu_{\hat{f}_{pdp}(x_{j})}\right)^{2}\right], 
\end{equation}
where 
\begin{equation}
    \mu_{\hat{f}_{pdp}(x_{j})} = \mathbb{E}_{x_{j}} \left[\hat{f}_{pdp}(x_{j})\right].
\end{equation}
The PDP feature importance is then written as~\cite{greenwell2018simple}:
\begin{equation}
    \label{eq:GSA_PDP}
   I_{pdp}\left(x_{j}\right)= \sqrt{\text{ Var}_{x_{j}}[\hat{f}_{pdp}(x_{j})]}\approx\sqrt{\frac{1}{K-1} \sum_{k=1}^K\left(\hat{f}_{pdp}\left(x_{j}^{(k)}\right)-\frac{1}{K} \sum_{l=1}^K \hat{f}_{pdp}\left(x_j^{(l)}\right)\right)^2}, 
\end{equation}
where $K$ is the number of samples drawn from $\pi_{x_{j}}(x_{j})$. 

\(I_{pdp}\)  corresponds exactly to the (unnormalised) first‑order effect in the Hoeffding–Sobol' (functional ANOVA) decomposition~\cite{Sobol2001global}. However, it should not be interpreted as a proportion of total variance unless normalised by the total variance itself, because it only measures the spread of the conditional mean. Moreover, since PDP is computed by averaging model predictions over the complementary variables, its curve can blur interaction effects. This is precisely why ICE curves (and their dispersion) are needed when heterogeneity or interaction is suspected~\cite{goldstein2015peeking}.

A similar metric for assessing the interaction term can also be computed. Formally, let $\boldsymbol{x}_{S}=\{x_{i},x_{j}\}$ be the set of two variables that we have an interest in. First, build the joint partial dependence function $\hat{f}_{pdp}(x_{i},x_{j})$. From $\hat{f}_{pdp}(x_{i},x_{j})$, we can calculate $I_{pdp}(x_{i}|x_{j})$ for each value of $x_{j}$, in other words, we compute the partial dependence function of $x_{i}$ on the joint dependence function $\hat{f}_{pdp}(x_{i},x_{j})$. The standard deviation of $I_{pdp}(x_{i}|x_{j})$ is then computed, and the same is done for \textcolor{black}{$I_{pdp}(x_{j}|x_{i})$}. The standard deviation values for both $I_{pdp}(x_{i}|x_{j})$ and $I_{pdp}(x_{j}|x_{i})$ are then averaged to yield the interaction effect. Regardless, Eq.~(\ref{eq:GSA_PDP}) still represents only the average effect and might cause misunderstanding, which is why we introduce ICE-based feature importance, as explained in the next section.

\subsection{ICE-based feature importance}

In this paper, we address the limitations of PDP-based feature importance by introducing an ICE-based feature importance metric. While conceptually similar to PDP-based importance, the key distinction is that we first quantify the feature importance for each individual ICE curve before aggregating. Formally, the variance of a single ICE curve with respect to $x_{j}$, given the $i$-th realization of the complementary features in $\boldsymbol{\Omega}_{C}$ (denoted as $\boldsymbol{x}_{C}^{(i)}$), is calculated as follows:

\begin{equation}
    \mathrm{Var}_{x_{j}}\left[\hat{f}(x_{j},\boldsymbol{x}_{C}^{(i)})\right] = \mathbb{E}_{x_{j}}\left[\left(\hat{f}(x_{j},\boldsymbol{x}_{C}^{(i)})-\mu_{\hat{f}(x_{j},\boldsymbol{x}_{C}^{(i)})}\right)^{2}\right], 
    \label{eq10}
\end{equation}
where the mean of the curve is given by
\begin{equation}
    \label{eq_ice_curve_mean}
    \mu_{\hat{f}(x_{j},\boldsymbol{x}_{C}^{(i)})} = \mathbb{E}_{x_{j}}\left[ \hat{f}(x_{j},\boldsymbol{x}_{C}^{(i)}) \right].
\end{equation}
Consequently, the ICE-based importance for a specific instance is defined as the standard deviation of the curve:
\begin{equation}
    \label{eq_ice_curve_std}
    I_{\mathrm{ice}}\left(x_{j};\boldsymbol{x}_{C}^{(i)}\right) = \sqrt{\mathrm{Var}_{x_{j}}\left[\hat{f}(x_{j},\boldsymbol{x}_{C}^{(i)})\right]}.
\end{equation}
In practice, this is approximated empirically using $K$ grid points for feature $x_j$:
\begin{equation}
    \label{eq_ICE_importance}
    I_{\mathrm{ice}}\left(x_{j};\boldsymbol{x}_{C}^{(i)}\right) \approx \sqrt{\frac{1}{K-1} \sum_{k=1}^K\left(\hat{f}\left(x_{j}^{(k)},\boldsymbol{x}_{C}^{(i)}\right)-\frac{1}{K} \sum_{l=1}^K \hat{f}\left(x_{j}^{(l)},\boldsymbol{x}_{C}^{(i)}\right)\right)^2}.
\end{equation}

This yields a distribution of $I_{\mathrm{ice}}$ values corresponding to different realizations of $\boldsymbol{x}_{C}^{(i)}$. From this distribution, we derive two aggregate metrics:
\begin{equation}
    \label{eq_mean_ICE}
    \mu_{I_{\mathrm{ice}},x_{j}} = \mathbb{E}_{\boldsymbol{x}_{C}}\left[I_{\mathrm{ice}}(x_{j};\boldsymbol{x}_{C})\right] = \int_{\boldsymbol{\Omega}_{C}} I_{\mathrm{ice}}\left(x_{j};\boldsymbol{x}_{c}\right) \boldsymbol{\pi}_{\boldsymbol{x}_{C}}(\boldsymbol{x}_{c})d\boldsymbol{x}_{c},
\end{equation}
and
\begin{equation}
    \label{eq_var_ICE}
    \sigma_{I_{\mathrm{ice}},x_{j}}^2 = \mathrm{Var}_{\boldsymbol{x}_{C}}\left[I_{\mathrm{ice}}(x_{j};\boldsymbol{x}_{C})\right] = \int_{\boldsymbol{\Omega}_{C}} \left(I_{\mathrm{ice}}(x_{j};\boldsymbol{x}_{c})-\mu_{I_{\mathrm{ice}},x_{j}}\right)^{2} \boldsymbol{\pi}_{\boldsymbol{x}_{C}}(\boldsymbol{x}_{c})d\boldsymbol{x}_{c}.
\end{equation}
\textcolor{black}{In practice, these metrics are calculated using Monte Carlo simulations as follows:
\begin{equation}
\label{eq_mean_ICE_MC}
\mu_{I_{\mathrm{ice}},x_j}
\;\approx\; \hat{\mu}_{I_{\mathrm{ice}},x_{j}} =
\frac{1}{N}\sum_{i=1}^N
I_{\mathrm{ice}}\!\left(x_j;\boldsymbol{x}_C^{(i)}\right).
\end{equation}
and 
\begin{equation}
\label{eq_var_ICE_MC}
\sigma_{I_{\mathrm{ice}},x_j}^2
\;\approx\; \hat{\sigma}_{I_{\mathrm{ice}},x_j}^2=
\frac{1}{N-1}\sum_{i=1}^N
\left(
I_{\mathrm{ice}}\!\left(x_j;\boldsymbol{x}_C^{(i)}\right)
-
\hat{\mu}_{I_{\mathrm{ice}},x_j}
\right)^2,
\end{equation}
where $\{\boldsymbol{x}_C^{(i)}\}_{i=1}^N$ is sampled from $\boldsymbol{\pi}_{\boldsymbol{x}_C}$ and $N$ is the size of Monte Carlo samples. The proposed ICE-based feature importance method requires evaluating the predictive model on a grid of $K$ values for the feature of interest $x_j$ across $N$ realizations of the complementary inputs $\boldsymbol{x}_C$. As a result, the computational cost for a single feature scales as $\mathcal{O}(NK)$ model evaluations (the aggregation itself is negligible), and as $\mathcal{O}(mNK)$ when applied to all $m$ input features. This is identical to the computational cost of computing partial dependence plots, since PDPs are obtained by averaging the same set of ICE curves. To put it into perspective, the computational complexity of total Sobol' indices is $\mathcal{O}(mN)$ using the standard Saltelli estimator for all inputs~\cite{saltelli2010variance}. In this regard, the proposed method incurs a higher computational cost compared to total Sobol' indices in exchange for resolving the conditional response structure rather than solely attributing output variance.}

These two metrics serve distinct but complementary purposes. The first metric, $\mu_{I_{\mathrm{ice}},x_{j}}$, measures global importance. Unlike PDP-based importance, which averages the predictions (often leading to cancellation effects), $\mu_{I_{\mathrm{ice}},x_{j}}$ averages the magnitude of variations, thereby mitigating the issue of heterogeneous effects canceling each other out. Conversely, the second metric, $\sigma^2_{I_{\mathrm{ice}},x_{j}}$, assesses the influence of interactions. A high dispersion in $I_{\mathrm{ice}}$ indicates that the impact of $x_j$ varies significantly depending on the values of the complementary features $\boldsymbol{x}_C$, which indicates strong feature interactions. These metrics can be visualized simultaneously using a bar plot for $\mu_{I_{\mathrm{ice}},x_{j}}$ with $\sigma_{I_{\mathrm{ice}},x_{j}}$ represented as error bars.

\textcolor{black}{The key difference between the proposed $\mu_{I_{\mathrm{ice}},x_j}$ and the total Sobol' index lies in what aspect of input influence they quantify. The metric $\mu_{I_{\mathrm{ice}},x_j}$ characterizes how the input variable $x_j$ influences the model response across different conditional settings by measuring the average conditional effect of $x_j$ over all realizations of the remaining inputs. In contrast, the total Sobol' index quantifies the fraction of the output variance attributable to $x_j$, including all interaction effects with other variables. While both metrics account for interactions between $x_j$ and the remaining inputs, they do so in fundamentally different ways: $\mu_{I_{\mathrm{ice}},x_j}$ focuses on conditional response behavior, whereas the total Sobol' index focuses on variance decomposition. The two metrics can lead to different variable rankings, as illustrated by the first demonstration in Section~\ref{sec:first_demonstration}. On the other hand, $\sigma_{I_{\mathrm{ice}},x_{j}}^2$ is an additional metric that complements $\mu_{I_{\mathrm{ice}},x_j}$ by capturing how heterogeneous this effect is with respect to variations in the remaining inputs. This variability is closely related to interaction effects in the Sobol' sense. The difference is that $\sigma_{I_{\mathrm{ice}},x_{j}}^2$ quantifies the variance of each ICE curve across different conditional settings, while Sobol' interaction indices quantify the contribution of interaction terms to the overall output variance. As such, $\sigma_{I_{\mathrm{ice}},x_{j}}^2$ reflects heterogeneity in conditional response behavior, whereas Sobol' interaction indices reflect variance attribution within a functional ANOVA decomposition. }

Equation~\eqref{eq_ice_curve_mean} defines the mean of the ICE curve for instance \(\boldsymbol{x}_{C}^{(i)}\), and Eq.~\eqref{eq_ice_curve_std} is its standard deviation. When we then average and compute the spread of \(I_{\mathrm{ice}}(x_j;\boldsymbol{x}_C^{(i)})\) across instances, we obtain interpretable quantities in the same units as the model output. One could instead use the conditional variance itself, without taking a square root, as the instance-wise importance measure. That choice has a formal appeal: by the law of total variance,
\(
\mathrm{Var}(y)
= \mathrm{Var}_{x_j}\bigl(\mathbb{E}[y \mid x_j]\bigr)
+ \mathbb{E}_{x_j}\bigl[\mathrm{Var}(y \mid x_j)\bigr],
\)
and in classical Sobol' analysis (assuming input independence), one identifies the first-order component with
\(\mathrm{Var}_{x_j}(\mathbb{E}[y\mid x_j])\) and the complementary term \(\mathbb{E}[\mathrm{Var}(y\mid x_j)]\) as the residual / interaction-plus component. This decomposition underlies how Sobol' indices are defined~\cite{iooss2019shapley}.
By contrast, standard deviations are more interpretable (same units as \(y\)), but do not decompose additively in the same way. Thus, the two metrics are complementary, and we recommend using the variance for variance decomposition and additivity, while the standard deviation is of better use for interpretability in the original unit of model output. 

\textcolor{black}{The pseudocode for the calculation of ICE-based feature importance is shown in Algorithm~\ref{alg:ICE_FI}. Note that the calculation of the mean ICE-based importance $\mu_{I_{\mathrm{ice}},x_j}$ and variance of ICE-based importance $I_{\mathrm{ice}}\left(x_{j};\boldsymbol{x}_{C}^{(i)}\right)$ are done by using Monte Carlo simulation. }
\begin{algorithm}
\caption{\textcolor{black}{ICE-based Feature Importance for a variable $x_j$}}
\begin{algorithmic}
\Require Trained model $\hat f(\boldsymbol{x})$, feature of interest $x_j$, $K$ grid points $\{x_j^{(k)}\}_{k=1}^K$
\Ensure $\mu_{I_{\mathrm{ice}},x_j}$, $\sigma^2_{I_{\mathrm{ice}},x_j}$

\State Sample $\{\boldsymbol{x}_C^{(i)}\}_{i=1}^N$ from $\boldsymbol{\pi}_{\boldsymbol{x}_C}$ 

\For{$i = 1,\dots,N$}
\State Evaluate the $i$-th ICE curve $\hat{f}(x_{j},\boldsymbol{x}_{C}^{(i)})$
    \State Estimate the standard deviation of the $i$-th ICE curve $I_{\mathrm{ice}}\left(x_{j};\boldsymbol{x}_{C}^{(i)}\right)$ using Eq.~(\ref{eq_ICE_importance})
\EndFor

\State Estimate mean ICE-based importance $\mu_{I_{\mathrm{ice}},x_j}$ using Eq.~(\ref{eq_mean_ICE})
\State Estimate variance of ICE-based importance $\sigma^2_{I_{\mathrm{ice}},x_j}$ using Eq.~(\ref{eq_var_ICE})
\end{algorithmic}
\label{alg:ICE_FI}
\end{algorithm}

%TO DO PAUL

% Let \(y = f(\boldsymbol{x})\) denote the scalar (or vector) model output, and partition the inputs into \(x_j\) (feature of interest) and the complementary vector \(\boldsymbol{x}_c\). 
\color{black}
Recent work has established a rigorous connection between conditional variability summaries and optimal-transport–based GSA~\cite{borgonovo2025global,chiani2025global}. 
Here we treat \(x_j\) as the feature and \(x_j^{(i)}\) as a specific sampled value of that feature. Denote the unconditional and conditional output distributions by \(\mathcal{L}(y)\) and \(\mathcal{L}(y \mid x_j = x_j^{(i)})\), respectively. 
Define the unconditional mean and covariance as $\mu_y = \mathbb{E}[y]$ and  
$\Sigma_y = \mathrm{Var}(y)$, respectively.
Also, denote the conditional mean and covariance (pointwise in \(x_j\)) as $
\mu_{y \mid x_j} = \mathbb{E}[y \mid x_j]$ and  
$\Sigma_{y \mid x_j} = \mathrm{Var}(y \mid x_j)$, respectively. Chiani \textit{et al.}~\cite{chiani2025global} propose a sensitivity index based on the squared 2–Wasserstein distance, \(W_2^2\), between \(\mathcal{L}(y)\) and \(\mathcal{L}(y \mid x_j)\). Under Gaussian or elliptical assumptions, this distance admits the closed-form Wasserstein–Bures decomposition
\begin{equation}
W_2^2\big(\mathcal{L}(y), \mathcal{L}(y \mid x_j)\big)
= \underbrace{\| \mu_y - \mu_{y \mid x_j} \|_2^2}_{\text{mean shift}}
+ \underbrace{\mathrm{Tr}\!\left(\Sigma_y + \Sigma_{y \mid x_j} - 2 \big(\Sigma_y^{1/2} \Sigma_{y \mid x_j} \Sigma_y^{1/2}\big)^{1/2}\right)}_{\text{covariance shift}}.
\end{equation}

% The first term measures how the conditional mean \(\mu_{y\mid x_j=x_j}\) departs from the unconditional mean \(\mu_y\) (the mean–shift contribution), while the second (Bures) term quantifies changes in dispersion / covariance structure induced by conditioning on \(x_j=x_j\).
% Recent work has established a rigorous connection between conditional variability summaries and optimal transport-based GSA~\cite{borgonovo2025global, chiani2025global}. Specifically, in~\cite{chiani2025global}, the authors propose a sensitivity index based on the squared 2-Wasserstein distance ($W_2^2$) to quantify the discrepancy between the unconditional model output distribution $\mathcal{L}(y)$ and the conditional distribution $\mathcal{L}(y|x_j)$~\cite{chiani2025global}.
% Under the assumption that the distributions are Gaussian or elliptical, this distance admits a closed-form decomposition known as the Wasserstein-Bures semimetric:
% \begin{equation}
%     W_2^2\big(\mathcal{L}(y), \mathcal{L}(y|x_j)\big) = \underbrace{\| \mu_y - \mu_{y|x_j} \|_2^2}_{\text{Mean Shift}} + \underbrace{\text{Tr}\left( \Sigma_y + \Sigma_{y|x_j} - 2(\Sigma_y^{1/2} \Sigma_{y|x_j} \Sigma_y^{1/2})^{1/2} \right)}_{\text{Covariance Shift}}
% \end{equation}
This decomposition highlights two distinct sources of feature importance:
\begin{enumerate}
    \item The shift in the expected value (first term), which corresponds to the traditional variance-based sensitivity measures~\cite{lamboni2011multivariate,gamboa2014sensitivity}.
    \item The change in the distribution's shape or dispersion (second term), which captures effects on the output variance and covariance structure.
\end{enumerate}

\color{black}
Our proposed ICE-based metrics are conceptually related to the Wasserstein–Bures framework in that they distinguish between average effects and variability effects; however, the correspondence is not algebraic in general. 
In the Wasserstein–Bures decomposition, the covariance shift term is defined through the conditional covariance \(\Sigma_{y \mid x_j}\), that is, the variability of \(y\) across the complementary inputs \(\boldsymbol{x}_C\) for a fixed \(x_j\). 
In contrast, the ICE-based dispersion quantifies the variability of \(\hat f(x_j, \boldsymbol{x}_C^{(i)})\) across different values of \(x_j\) for a fixed realization \(\boldsymbol{x}_C^{(i)}\). 
Thus, the two metrics probe variability along different conditional directions: \(\Sigma_{y \mid x_j}\) looks across \(\boldsymbol{x}_C\) at fixed \(x_j\), whereas ICE-based dispersion looks across \(x_j\) at fixed \(\boldsymbol{x}_C^{(i)}\).
The ICE-based dispersion should therefore be interpreted as a complementary measure of response heterogeneity rather than a direct estimator of the Wasserstein covariance shift. 
Under additional structural assumptions (such as additive or separable models, or finite orthonormal expansions like PCE), these perspectives can be related. 
A direct Wasserstein-based dispersion measure would involve the conditional standard deviation 
\(\sigma_{y \mid x_j} = \sqrt{\mathrm{Var}(y \mid x_j)}\), 
which quantifies variability across \(\boldsymbol{x}_C\) for a fixed \(x_j\), 
whereas ICE-based dispersion captures the variability of \(\hat f(x_j, \boldsymbol{x}_C^{(i)})\) 
across different values of \(x_j\) for a fixed realization \(\boldsymbol{x}_C^{(i)}\).

\color{black}
Note that the variability captured by $I_{\mathrm{ice}}\left(x_{j};\boldsymbol{x}_{C}^{(i)}\right)$ will always account for more localized variations specific to $\boldsymbol{x}_{C}^{(i)}$, which are not captured by the globally averaged $\hat{f}_{pdp}(x_j)$. As a result, we demonstrate, in a certain setting, that:
\begin{equation}
    \label{eq:ICE_ineq}
    \mathbb{E}_{\boldsymbol{x}_C}\left[ I_{\mathrm{ice}}(x_j; \boldsymbol{x}_C) \right] \ge I_{pdp}(x_j).
    \end{equation}
More precisely, we demonstrate that the inequality holds for several classes of functions, most notably for those expressible as a finite expansion of a linear combination of bases and functions with some combination of additive and multiplicative separability, as detailed in Appendix A. Since the PCE surrogate model itself constitutes a finite expansion of such bases, the inequality is satisfied when PCE is employed. 

\textcolor{black}{The inequality in Eq.~(\ref{eq:ICE_ineq}) may appear analogous to the well-known inequality in the Sobol' index, namely that the total Sobol' index is always greater than or equal to the first-order Sobol' index. However, the two results arise from fundamentally different mechanisms. The Sobol' inequality follows directly from the variance-based decomposition, whereby the total index aggregates the main effect and all interaction effects in an additive manner. In contrast, the ICE-based inequality reflects the fact that averaging conditional response variations preserves heterogeneous effects that may cancel out when the model response is averaged first, as in the PDP.}

\subsection{ICE-based correlation values}

%  Such insights are invaluable in design exploration, guiding more informed decisions on controlling design parameters during optimization.

Beyond quantifying GSA, it is also essential to assess how interactions modify the relationship between inputs and predictions. This enables the identification of variables that change from positive to negative associations when interacting with others. To complement the ICE-based feature importance, we introduce the {ICE-based correlation values} which quantify the correlation between ICE curves and the partial dependence function:
\begin{equation}
\label{eq_ice_corr}
\rho(x_{j},{\boldsymbol{x}_{C}^{(i)}})=\text{Corr}(\hat{f}(x_{j},\boldsymbol{x}_{C}^{(i)}),\hat{f}_{pdp}(x_{j})),
\end{equation}
where $\text{Corr}$ denotes the correlation.  That is, we compute the correlation between the ICE curve $\hat{f}(x_{j},\boldsymbol{x}_{C}^{(i)})$ for a specific instance $\boldsymbol{x}_{C}^{(i)}$  and the corresponding partial dependence function $\hat{f}_{pdp}(x_{j})$ for the variable $x_j$. \textcolor{black}{In this paper, we measure the linear relationship between ICE curves and PDP using the Pearson correlation coefficient to ease interpretation. Alternative methods, including Spearman correlation, could be used to detect monotonic dependencies.}

From this, we define the following quantity:
\begin{equation}
\label{eq_ice_corr_std}
    \sigma_{\rho_{x_{j}}} = \sqrt{\mathbb{E}_{\boldsymbol{x}_{C}}\left[\left(\rho(x_{j},\boldsymbol{x}_{C}) - \mathbb{E}_{\boldsymbol{x_{C}}}\left[\rho(x_{j},\boldsymbol{x}_{C})\right]\right)^2\right]},
\end{equation}
which is essentially the standard deviation of the ICE correlation values, and can be computed from a Monte Carlo simulation. If $\sigma_{\rho}=0$, this indicates that the partial dependence function is similar to ICE curves (in other words, there is no interaction). Conversely, a high value of $\sigma_{\rho}$ suggests that most of the ICE curves differ significantly from the partial dependence line, indicating a strong interaction effect on the trend. Additionally, showing the boxplot or \textcolor{black}{violin plot} of {ICE-based correlation values} also helps in the assessment. \textcolor{black}{It is worth noting that, although Pearson correlation is used, the metric $\sigma_{\rho_{x_j}}$ is capable of detecting nonlinear interactions in terms of trend shifts, as it quantifies the variability in correlation between the PDP and multiple ICE curves.}

\textcolor{black}{The numerical procedure is shown in Algorithm~\ref{Alg:ICE_correlation}. The computation of the ICE-based correlation metric in Eq.~(\ref{eq_ice_corr}) requires evaluating the Pearson correlation between an ICE curve and the corresponding PDP, both discretised on a grid of $K$ points. Computing a single correlation therefore scales as $\mathcal{O}(K)$. Repeating this procedure across $N$ realizations of the complementary inputs $\boldsymbol{x}_C$ results in an overall complexity of $\mathcal{O}(NK)$ per input variable. The subsequent computation of the mean and standard deviation in Eq.~(\ref{eq_ice_corr_std}) incurs only $\mathcal{O}(N)$ additional cost and is negligible in comparison. When applied to all $m$ input input variables, the total computational complexity scales as $\mathcal{O}(mNK)$.}

\begin{algorithm}
\caption{\textcolor{black}{ICE-based correlation values for a variable $x_j$}}
\begin{algorithmic}
\Require Trained model $\hat f(\boldsymbol{x})$, input variable of interest $x_j$, $K$ grid points $\{x_j^{(k)}\}_{k=1}^K$, partial dependence function $\hat{f}_{pdp}(x_j)$
\Ensure $\sigma_{\rho_{x_{j}}}$

\State Sample $\{\boldsymbol{x}_C^{(i)}\}_{i=1}^N$ from $\boldsymbol{\pi}_{\boldsymbol{x}_C}$ 

\For{$i = 1,\dots,N$}
\State Evaluate the $i$-th ICE curve $\hat{f}(x_{j},\boldsymbol{x}_{C}^{(i)})$
\State Compute the correlation between $\hat{f}(x_{j},\boldsymbol{x}_{C}^{(i)})$ and $\hat{f}_{pdp}(x_j)$, i.e., $\rho(x_{j},{\boldsymbol{x}_{C}^{(i)}})$ from Eq.~(\ref{eq_ice_corr})
\EndFor

\State Compute the standard deviation of ICE correlation values $\sigma_{\rho_{x_{j}}}$ from Eq.~(\ref{eq_ice_corr_std})
\end{algorithmic}
\label{Alg:ICE_correlation}
\end{algorithm}
\textcolor{black}{It is important to note that the assessment of the ICE-based correlation value is meaningful only when the partial dependence function exhibits non-zero variance. In cases where the PDP is flat, indicating the absence of a main effect of the corresponding input variable, the Pearson correlation is mathematically undefined due to zero variance. However, in practical applications, a PDP that is theoretically flat may exhibit small but non-zero variance due to the numerical averaging of heterogeneous ICE curves or surrogate approximation errors. To avoid spurious interpretations, the proposed ICE-based correlation metrics are interpreted only when the following criterion is satisfied:
\begin{equation}
    \label{eq:ICE_corr_criterion}
    \frac{I_{pdp}}{\mu_{I_{ice}}} < \varepsilon,
\end{equation}
where $\varepsilon$ is the threshold criterion. Since $I_{pdp}$ quantifies the importance according to PDP, while $\mu_{I_{ice}}$ captures the average conditional variability across ICE curves, this criterion identifies interaction-dominated variables even when small numerical PDP variations are present. In such cases, ICE-based correlation metrics are not interpreted. Regardless, when this condition is satisfied, it indicates a substantial shift in the conditional trends by which the input variable influences the output, warranting a closer examination of the corresponding ICE curves. In this paper, we use $\varepsilon=5 \times10^{-3}$, such that the PDP-based importance must be less than 0.5\% of the average ICE-based importance.}

% In such cases, the Pearson correlation can be numerically computed but does not necessarily indicate a meaningful directional relationship. 

\textcolor{black}{\subsection{Evaluation workflow}
%In practice, the proposed sensitivity analysis framework is applied through a combination of visual diagnostics and quantitative ICE-based metrics. As a prerequisite, the model response function must first be available, either in the form of the original deterministic model or through a sufficiently accurate surrogate model. This requirement of accurate surrogate  model is important, since sensitivity measures derived from a surrogate model that does not adequately reproduce the underlying system behavior may be misleading and unreliable~\cite{razavi2015we}.}
%We acknowledge variance-based sensitivity analysis has also been successfully combined with Monte Carlo and importance sampling strategies in reliability assessment using Gaussian process surrogates, notably to quantify the contribution of metamodel and sampling uncertainties and to derive adaptive stopping criteria for failure probability estimation~\cite{menz2021variance}.
In practice, the proposed sensitivity analysis framework is applied through a combination of visual diagnostics and quantitative ICE-based metrics. As a prerequisite, the model response function must first be available, either in the form of the original deterministic model or through a sufficiently accurate surrogate model. This requirement is essential, since sensitivity measures derived from a surrogate that does not adequately reproduce the underlying system behavior may be misleading and unreliable~\cite{razavi2015we}.}
%
%Variance-based sensitivity analysis has also been successfully combined with Monte Carlo and importance sampling strategies in reliability assessment using Gaussian process surrogates, notably to quantify the contribution of metamodel and sampling uncertainties and to derive adaptive stopping criteria for failure probability estimation~\cite{menz2021variance}. This highlights the broader role of surrogate-assisted sensitivity analysis in computationally demanding engineering contexts.

% In practice, we use its leave-one-out version, $R^2_{\mathrm{LOO}}$, where the prediction at each sample location is obtained from a surrogate trained using the $n-1$ samples, thereby providing an estimate of the model’s generalization performance: 
%\begin{equation}
 %   R^2_{\mathrm{LOO}}
%= 1 - \frac{\sum_{i=1}^{n} \left(f(\boldsymbol{x}^{(i)}) - \hat{f}^{(-i)}(\boldsymbol{x}^{(i)})\right)^2}
%{\sum_{i=1}^{n} \left(f(\boldsymbol{x}^{(i)}) - \bar{f}\right)^2},
%\end{equation}
%where $\hat{f}^{-i}(\boldsymbol{x})$ denotes the surrogate trained with the $i$-th sample excluded. For the PCE surrogate model, the leave-one-out prediction error can be evaluated analytically without retraining the model~\cite{blatman2011adaptive}. 

\textcolor{black}{In this paper, we use the coefficient of determination, $R^2$,  to assess the accuracy of the surrogate model:
\begin{equation}
R^2
= 1 - \frac{\sum_{i=1}^{n_v} \left(f(\boldsymbol{x}^{(i)}) - \hat{f}(\boldsymbol{x}^{(i)})\right)^2}
{\sum_{i=1}^{n_v} \left(f(\boldsymbol{x}^{(i)}) - \bar{f}\right)^2},
\end{equation}
where $\bar{f} = \frac{1}{n_v}\sum_{i=1}^{n_v} f(\boldsymbol{x}^{(i)})$ and $n_v$ is the size of validation samples. The use of a surrogate model introduces additional uncertainty into the computed GSA metrics due to surrogate approximation error. Quantifying and propagating this uncertainty is beyond the scope of the present work and will be addressed in future studies. In the present analysis, we therefore require the surrogate model to be sufficiently accurate so that the extracted sensitivity information remains faithful to the underlying response behavior. To this end, we require the $R^2$ to exceed 0.95 to ensure a faithful representation of the extracted knowledge~\cite{marrel2024probabilistic}. }

\textcolor{black}{The analysis begins with visualization by inspecting PDPs and the corresponding ICE curves to obtain an initial qualitative understanding of global trends and conditional behavior. Subsequently, the ICE-based importance metric $\mu_{I_{\mathrm{ice}},x_j}$ is computed to assess the overall importance of each input variable in terms of the average magnitude of its conditional effect. The variability metric $\sigma^2_{I_{\mathrm{ice}},x_j}$ is then analyzed to quantify heterogeneity across conditional settings; a large value of $\sigma^2_{I_{\mathrm{ice}},x_j}$ indicates that the influence of $x_j$ strongly depends on the remaining inputs and suggests the presence of interaction effects. In such cases, the ICE curves are revisited to examine how the interaction manifests across different realizations of the complementary variables. Finally, the correlation between individual ICE curves and the corresponding PDP is evaluated to assess whether the average PDP is representative of the conditional response behavior. % This procedure provides a systematic workflow that combines visualization and quantitative diagnostics to interpret input importance and interaction effects under different scenarios.
}

\textcolor{black}{It should be noted that $\sigma^2_{I_{ice}}$ does not directly identify which input variables an input $x_j$ strongly interacts with. Such an interaction structure can be explored qualitatively through visualisation, for example, by plotting ICE curves coloured by another input variable, where clear trends indicate potential interactions. However, these visual diagnostics do not provide a quantitative measure of interaction magnitude. In such cases, variance-based interaction metrics, such as second-order Sobol' indices, can be employed to quantify pairwise interaction effects explicitly. Subsequently, the ICE curves are colour-coded according to the variable with which strong interactions are observed. Therefore, the proposed ICE-based metrics are complementary to other metrics, such as Sobol' indices, in which the former focuses on behavioral insight from conditional response analysis.}

\textcolor{black}{We emphasize that the primary focus of this work is the development of sensitivity metrics themselves, rather than uncertainty quantification of surrogate models. In this context, surrogate models are employed as computational enablers to emulate expensive simulations and to facilitate the application of the proposed sensitivity analysis framework. As such, the proposed metrics are applicable to any function or a surrogate model. When surrogate models are used, additional sources of uncertainty associated with model approximation and training data inevitably arise. While quantifying and propagating these surrogate-induced uncertainties is an important topic, it is beyond the scope of the present study. In this work, we ensure that sufficiently accurate surrogate models are available, allowing us to focus on extracting reliable sensitivity insights from the learned response. In practical settings, surrogate uncertainty can be assessed using established techniques such as bootstrapping, which can be naturally integrated with the proposed framework in future extensions.
}

\subsection{First demonstration}
\label{sec:first_demonstration}
\textcolor{black}{The aim of this first demonstration is to demonstrate the proposed evaluation workflow, highlighting how visualization and ICE-based metrics are used together on a simple problem.} At this stage, no surrogate model is used; instead, the methods are applied directly to the actual function. We conducted a comparison with {first-order, second-order, and total order Sobol'} indices, and averaged SHAP (referred to as $S_{F}$, $S_{T}$, and $\bar{Sh}$, respectively). For the analysis of two-way interactions, we compared the corresponding metrics for two-way interactions, denoted as $I_{pdp,ij}$, $S_{ij}$, and $\bar{Sh}_{ij}$, where $i$ and $j$ represent the indices of the variables, with $i \neq j$. 

% Normalized Shapley effects lie between first-order and total Sobol' indices~\cite{iooss2019shapley}. Accordingly, Sobol' and Shapley derived indices together provide transparent, extensible variance-based diagnostics of input effects and interactions for practical design in engineering. (Pram: Normalized Shapley effects are not averaged SHAP values, I hide this for now)

The three-variable analytical problem is adopted from Goldstein et. al.~\cite{goldstein2015peeking}, written as:
\begin{equation}
 y=0.2 x_1-5 x_2+10 x_2 \mathds{1}_{x_3 \geq 0}, 
\end{equation}
where $x_{i} \in [-1,1]$ for $i=1,2,3$.
The main challenge with this problem is that $x_{2}$ strongly interacts with $x_{3}$ in such a way that the PDP yields a flat curve (therefore, the $I_{pdp}$ is zero or close to zero). We did not perform a thorough analysis of this problem; instead, we use this problem to check the adequacy of the new metrics ({\textit{i.e.}}, $\mu_{I_{ice}}$, $\sigma_{I_{ice}}$, and $\sigma_{\rho}$). The {Sobol'} indices and averaged SHAP were calculated using $10^{6}$ samples. Due to the low dimensionality of the problem, the SHAP values are calculated analytically. To calculate the PDP and ICE-based metrics, 10,000 ICE curves were drawn from the actual function {but the SMT-Explainability toolbox allows us to do the same based on a surrogate model~\cite{saves2023smt,robani2025smt}.}

The anchored PDP and ICE plots for this problem are shown in Fig.~\ref{fig:threevar_PDP}. \textcolor{black}{In this figure, as well as in similar figures throughout the paper, the one-dimensional PDP is represented by a dashed line, whereas the corresponding ICE curves are shown as solid lines.} It is evident that the influence of $x_{2}$ on $y$ varies depending on the value of $x_{3}$, and vice versa. As a result, the PDPs for $x_{2}$ and $x_{3}$ appear as almost flat lines, leading to a near-zero PDP-based input importance for both variables. However, this is misleading, as the true importance of these variables is not zero. On the other hand, $x_{1}$ does not interact with $x_{2}$ and $x_{3}$, which is why its PDP and ICE curves are essentially identical. Moreover, its overall impact is actually less significant than that of $x_{2}$ and $x_{3}$.

\begin{figure}[hbt!]
	\centering
	\begin{subfigure}{.32\columnwidth}
		\includegraphics[width=1\columnwidth]{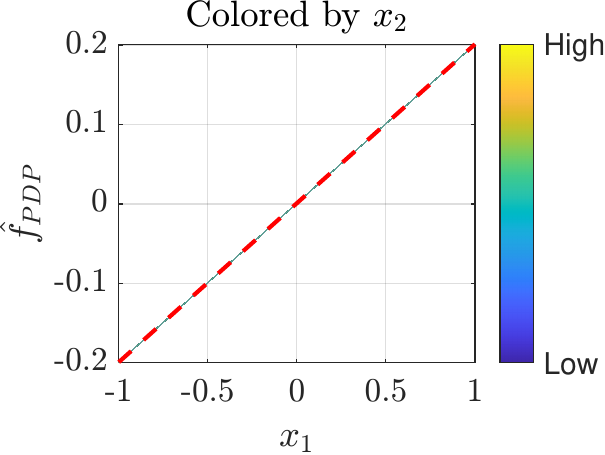}%
		\label{fig:threevar_PDP_X1_with_x2.eps}
    \caption{$x_{1}$}
    \end{subfigure}
	\begin{subfigure}{.32\columnwidth}
		\includegraphics[width=1\columnwidth]{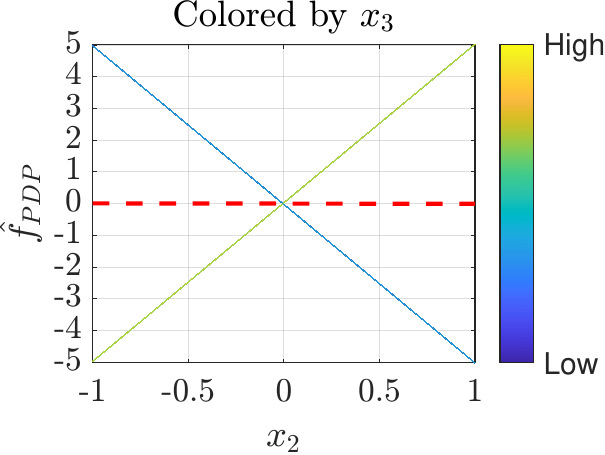}%
		\caption{$x_{2}$}%
		\label{fig:threevar_PDP_X2_with_x3.eps}
    \end{subfigure}
	\begin{subfigure}{.32\columnwidth}
		\includegraphics[width=1\columnwidth]{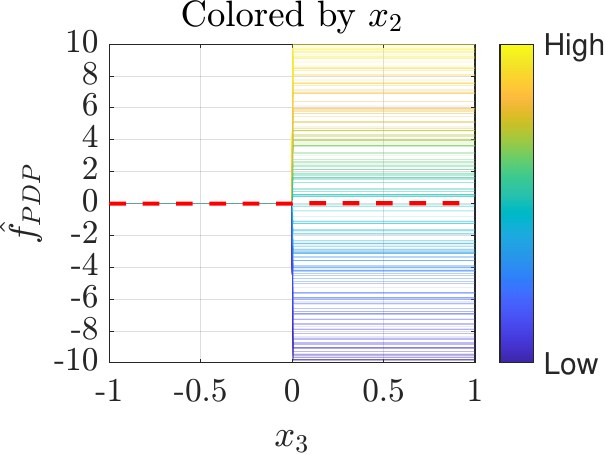}%
		\caption{$x_{3}$}%
		\label{fig:threevar_PDP_X3_with_x2.eps}
    \end{subfigure}
        \caption{PDP and ICE plots for the three-variable problem. }
	\label{fig:threevar_PDP}
\end{figure}

The complete GSA results are presented in Table~\ref{tbl:result_3var_GSA}. The first-order {Sobol'} indices are very small for all variables (they should essentially be zero for $x_2$ and $x_3$ due to their flat PDPs). However, the total {order Sobol'} indices are close to 1 for $x_{2}$ and $x_{3}$, suggesting that their only contribution to the total variance comes from the interaction between $x_{2}$ and $x_{3}$. As anticipated, $I_{pdp}$ does not suggest that $x_{2}$ and $x_{3}$ are more significant than $x_{1}$, due to the previously mentioned challenge. In contrast, Averaged SHAP accurately identifies the importance of $x_{2}$ and $x_{3}$. Unlike $I_{pdp}$, $\bar{Sh}$ is based on the impact of each input on the model's predictions and is capable of detecting subtle yet crucial relationships that $I_{pdp}$ might overlook. It can also be observed that ICE-based feature importance yields finite and reasonable values for $x_{2}$ and $x_{3}$. Notably, $\mu_{I_{\mathrm{ice}}}$ effectively identifies the importance of $x_{2}$ and $x_{3}$ by focusing on the importance of each ICE curve. Thus, it can be concluded that $\mu_{I_{\mathrm{ice}}}$ addresses the main limitation of $I_{pdp}$. The $\sigma_{I_{\mathrm{ice}}}$ value that equals 0 for $x_{1}$ is the consequence of its non-interacting nature (note that we analyze the standard deviation instead of the variance of $I_{\mathrm{ice}}$). However, $\sigma_{I_{\mathrm{ice}}}=0$ can indicate either (i) a genuinely non-interacting additive input, as in the case of $x_1$, or (ii) the presence of interactions that alter curve shape (\textit{e.g.}, flip sign or slope) while preserving the integrated spread. Hence, in this second case, visual inspection of ICE curves or slope-based diagnostics is required to disambiguate these cases, especially when relying on surrogate models.
\textcolor{black}{Indeed, for the case of $x_{2}$, this means that the importance magnitude of each ICE curve of $x_{2}$ is the same, although the gradient shifts depending on how it interacts with $x_{3}$. Distinguishing between these two scenarios requires an explicit examination of interaction metrics, which are not captured by $\sigma_{I_{\mathrm{ice}}}$. In fact, as shown in Table~\ref{tbl:result_3var_GSA_int}%(discussed in detail later)
, $x_2$ is found to strongly interact with $x_3$. Conversely, the table also shows that $x_1$ does not interact with other variables.} Lastly, the finite value of $\sigma_{I_{\mathrm{ice}}}$ for $x_{3}$ is reasonable because, as can be observed in Fig.~\ref{fig:threevar_PDP_X3_with_x2.eps}, each ICE curve yields a different $I_{\mathrm{ice}}$.

\textcolor{black}{It is worth discussing the different rankings produced by the total Sobol' index $S_T$ and the proposed ICE-based importance metric $\mu_{I_{\mathrm{ice}}}$, as illustrated in this problem. The total Sobol' is a GSA index that assigns nearly identical importance to $x_2$ and $x_3$, reflecting the fact that the dominant contribution to the output variance arises from their interaction, which is fully attributed to both variables in the total-effect formulation. 
In contrast, $\mu_{I_{\mathrm{ice}}}$ is an averaged local index that ranks $x_2$ as more important than $x_3$, since it measures the mean conditional variability along ICE curves and therefore captures how strongly each variable influences the model response across different conditional settings. This example highlights that while $S_T$ quantifies variance contribution, $\mu_{I_{\mathrm{ice}}}$ characterizes conditional response behavior, leading to different but complementary importance rankings.}

\textcolor{black}{Lastly, $\sigma_{\rho}$ is not a measure of GSA but a change in the trend of ICE curves with respect to the PDP. The fact that $\sigma_{\rho}=0$ for $x_{1}$ suggests that $x_{1}$ does not interact with other variables. Furthermore, the values of $\sigma_{\rho}$ are not calculated for $x_{2}$ and $x_{3}$ because it does not satisfy the criterion mentioned in Eq.~(\ref{eq:ICE_corr_criterion}). In particular, this signals a flat PDP and a very strong interaction. Indeed, as can be seen in Fig.~\ref{fig:threevar_PDP}, the PDPs for $x_2$ and $x_3$ are essentially flat, while the ICE curves indicate a significant change of trends.}

\begin{table}[hbt!]
\centering
\caption{GSA Results for the 3-variable problem.}
\label{tbl:result_3var_GSA}
\begin{tabular}{|c|c|c|c|}
\hline
\textbf{Variable} & $x_{1}$ & $x_{2}$ & $x_{3}$ \\
\hline
$S_{F}$ & 0.001 & 0.004 & 0.001\\
$S_{T}$ & 0.003 & 0.997 & 0.998 \\
$\bar{Sh}$ & 0.552 & 1.750 & 0.750 \\
$I_{pdp}$  & 0.116 & 0.003 & 0.010 \\
$\mu_{I_{\mathrm{ice}}}$ & 0.116 & 2.908 & 2.505 \\
$\sigma_{I_{\mathrm{ice}}}$ & 0.000 & 0.000 & 1.449 \\
$\sigma_{\rho}$ & 0.000 & - & - \\
\hline
\end{tabular}
\end{table}

% $I_{ice,int}$ & 0 & 2.5125/2.9013 & 2.4858/4.0725\\

The two-way GSA results are presented in Table~\ref{tbl:result_3var_GSA_int}. It is noteworthy that all metrics (namely, two-way $I_{pdp}$, two-way $\bar{Sh}_{ij}$, and second-order {Sobol'} indices) successfully identify the interaction between $x_{2}$ and $x_{3}$. {It is an expected {behavior} as Second-order Sobol'~indices quantify how much the combined interaction of two input variables contributes to the total output variance, over and above the sum of their individual effects.}
In the case of $\mu_{I_{\mathrm{ice}}}$, this interaction information is incorporated into $\mu_{I_{\mathrm{ice}}}$ and complements it. Specifically, $\mu_{I_{\mathrm{ice}}}$ and its corresponding $\sigma_{I_{\mathrm{ice}}}$ reflect the overall impact of the interaction, while $I_{pdp,ij}$ quantifies the influence of specific interactions.

\begin{table}[hbt!]
\centering
\caption{Two-way GSA Results for the 3-variable problem.}
\label{tbl:result_3var_GSA_int}
\begin{tabular}{|c|c|c|c|}
\hline
\textbf{Variable} & $x_{1}-x_{2}$ & $x_{1}-x_{3}$ & $x_{2}-x_{3}$ \\ \hline
$I_{pdp,ij}$ & 0.000 & 0.000 & 0.771 \\
$S_{ij}$ & 0.004 & 0.004 & 0.993 \\
$\bar{Sh}_{ij}$ & 0.000 & 0.000 & 1.246 \\
\hline
\hline
\end{tabular}
\end{table}

\section{Demonstration on analytical and engineering problems}
Metrics for analytical problems were evaluated on the original function, while surrogate models were used for engineering cases where the true function was unknown. For all problems, we conducted a comparison with other sensitivity metrics (\textit{e.g.}, {Sobol'} indices and averaged SHAP). % In terms of visualization, we compared 

\subsection{5-variable Friedman function}

We first demonstrated and compared the capabilities of ICE and PDP on the 5-variable Friedman function~\cite{friedman1991multivariate}, written as:
\begin{equation}
    f(\boldsymbol{x}) = 10\text{ sin}(\pi x_{1}x_{2}) + 20 (x_{3}-0.5)^{2}+10x_{4}-5x_{5}
\end{equation}
where $x_{i} \in [0,1]$, for $i=1,\ldots,5$. The equation shows interaction only between $x_1$ and $x_2$; thus, explainability methods should be able to correctly capture this relationship.

\subsubsection{Global sensitivity analysis}
First, let us compute the GSA values from {Sobol'} indices, averaged SHAP, PDP, and ICE.  The comparison of various GSA metrics is presented in Table~\ref{tbl:result_GSA_friedman}. Notably, all metrics produce the same ranking of importance, with the most significant variable being $x_{4}$ followed by $x_{1}$ and $x_{2}$, with $x_{3}$ and $x_{5}$ are considered to be less significant. Interestingly, $I_{pdp}$ and $\mu_{I_{\mathrm{ice}}}$ provide identical importance values for $x_{3}$, $x_{4}$, and $x_{5}$. This occurs because the centered ICE curves for variables without interactions are essentially the PDP itself, leading to the mean of $I_{\mathrm{ice}}$ ({\textit{i.e.}}, $\mu_{I_{\mathrm{ice}}}$) being identical to $I_{pdp}$. Conversely, the $\mu_{I_{\mathrm{ice}}}$ values for $x_{1}$ and $x_{2}$ are higher than their $I_{pdp}$ counterparts, which is due to the interaction effects considered in the calculation of $\mu_{I_{\mathrm{ice}}}$. In this regard, if the value $I_{pdp}$ equals $\mu_{I_{\mathrm{ice}}}$, this signals that the interaction is not present. Conversely, when $\mu_{I_{\mathrm{ice}}} > I_{pdp}$, this signals that the interaction term is present since this means that the ICE curves are not exactly similar to the PDP in terms of the trend and impact. Note that this approach is similar to interpreting first-order and total {order Sobol'} indices. It is also evident in this problem that the values of \(S_{T}\) and \(S_{F}\) are identical for \(x_{3}\), \(x_{4}\), and \(x_{5}\), while \(S_{T}\) for \(x_{1}\) and \(x_{2}\) are higher than \(S_{F}\) for \(x_{1}\) and \(x_{2}\).

\begin{table}[hbt!]
\centering
\caption{GSA results for the {F}riedman problem.}
\label{tbl:result_GSA_friedman}
\begin{tabular}{|c|c|c|c|c|c|}
\hline
\textbf{Variable} & $x_{1}$ & $x_{2}$ & $x_{3}$ & $x_{4}$ & $x_{5}$\\
\hline
$S_{F}$ & 0.195 & 0.198 & 0.091 & 0.348 & 0.086 \\
$S_{T}$ & 0.272 & 0.271 & 0.091 & 0.348 & 0.086 \\
$\bar{Sh}$ & 2.278 &   2.278  &  1.707 &    2.503 &    1.256 \\
$I_{pdp}$ &  2.176&  2.175&   1.509&   2.908&    1.454 \\
$\mu_{I_{\mathrm{ice}}}$ &  2.415&  2.422 &   1.509&   2.908&    1.454 \\ \hline
$\sigma_{I_{\mathrm{ice}}}$ & 0.856 &  0.852  & 0.000 &  0.000 &  0.000 \\
$\sigma_{\rho}$ & 0.110 & 0.114 & 0.000 & 0.000 & 0.000 \\
\hline
\end{tabular}
\end{table}

% $COS$ & 1.182 & 1.192 & 0 & 0 & 0\\
\begin{table}[hbt!]
\centering
\caption{Two-way GSA results for the Friedman problem. The other interactions yield essentially zero importance.}
\label{tbl:result_GSA_friedman_int}
\begin{tabular}{|c|c|c|c|}
\hline
\textbf{Variable} & $x_{1}-x_{2}$ & $x_{1}-x_{4}$ & $x_{2}-x_{4}$ \\ \hline
$I_{pdp,ij}$ & 0.995 & 0.000 & 0.000 \\
$\bar{Sh}_{ij}$ & 0.585 & 0.000 & 0.000 \\
$S_{ij}$ & 0.076 & 0.001 & 0.001 \\
\hline
\hline
\end{tabular}
\end{table}

The barplot of $I_{pdp}$, $\mu_{I_{\mathrm{ice}}}$, and $\bar{Sh}$ for the {F}riedman regression problem is further shown in Fig.~\ref{fig:friedman_GSA}. Here, it can be seen that $\mu_{I_{\mathrm{ice}}}$, with the errorbar showing $\mu_{I_{\mathrm{ice}}} \pm \sigma_{I_{\mathrm{ice}}}$ provides more information than $I_{pdp}$ and $\bar{Sh}$. Again, a large vertical bar indicates a stronger effect of interactions, as evidenced in $x_{1}$ and $x_{2}$, but non-existent in other variables. Further, Fig.~\ref{fig:friedman_COR} depicts the violin plot of the correlation between the ICE curves and PDP.  For $x_{1}$ and $x_{2}$, the overall correlations are positive, but some ICE curves deviate significantly from the PDP (indicating strong interaction), particularly those with lower correlation values. The non-zero values of $\sigma_{\rho}$ also signify that the effect of the interaction changes the trend of the ICE curves. It is worth noting that a non-zero value of the correlation between the ICE curve and PDP indicates a change in nonlinearity or trend. The $\sigma_{\rho}$ value for $x_{3}$, $x_{4}$, and $x_{5}$ is exactly zero because the trend for the ICE curve is exactly similar to its PDP. Let us now analyze the interaction GSA-values as shown in Table~\ref{tbl:result_GSA_friedman_int}. The two-way PDP feature importance, SHAP, and second-order {Sobol'} indices (see Table~\ref{tbl:result_GSA_friedman_int}). further confirm the presence of the interaction between $x_{1}$ and $x_{2}$. {In this regard, Second-order Sobol' indices quantify the interaction effects between pairs of input variables on the output variance, while interaction SHAP values measure the specific contribution of interactions between inputs to a model's prediction.} It was observed that the estimated interaction metrics are zero or nearly zero for all cross terms but $x_{1}-x_{2}$. Note that this information should be digested together with an analysis of the visualization of PDP and ICE, as shown next.

\begin{figure}[hbt!]
	\centering
	\begin{subfigure}{.45\columnwidth}
		\includegraphics[width=1\columnwidth]{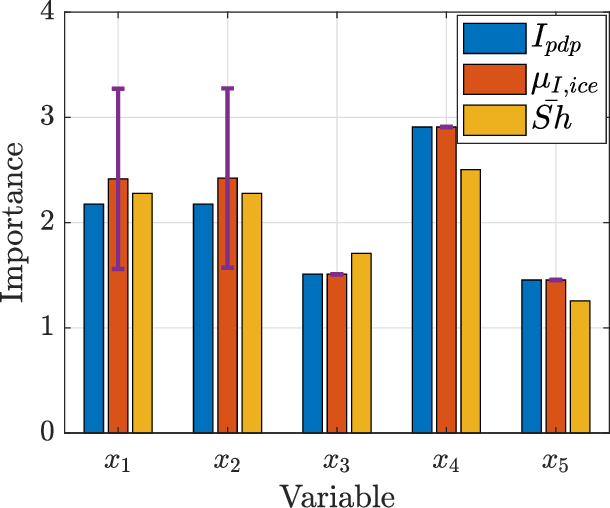}%
      \caption{Barplot of $I_{pdp}$, $\mu_{I_{\mathrm{ice}}}$, and $\bar{Sh}$.}
		\label{fig:friedman_GSA}
    \end{subfigure}
	\begin{subfigure}{.485\columnwidth}
		\includegraphics[width=1\columnwidth]{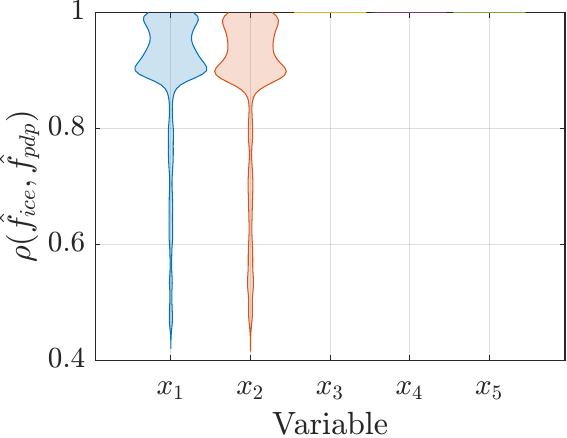}%
      \caption{\textcolor{black}{Violin plot of {ICE-based correlation values}}.}
		\label{fig:friedman_COR}
    \end{subfigure}
        \caption{Barplot of some GSA metrics and \textcolor{black}{violin plot} of ICE correlation values for the Friedman regression problem.}
	\label{fig:friedman_GSA_plot}
\end{figure}

\subsubsection{Visualization}

The first variable of interest is \( x_{1} \), with the centered PDP/ICE and SHAP plots shown in Figs.~\ref{fig:explainability_PDP_friedman_x1} and~\ref{fig:explainability_SHAP_friedman_x1}, respectively. The plots for \( x_{2} \) are similar, as both variables appear only in the term \( 10 \sin(\pi x_{1}x_{2}) \). The curves are colored by \( x_{2} \) and \( x_{4} \), where \( x_{1} \) interacts solely with \( x_{2} \). Notably, the centered PDP/ICE and SHAP plots reveal distinct insights: the effect of \( x_{1} \) varies with the value of \( x_{2} \), reflecting their interaction. For instance, when $x_{2}$ is small, the impact of $x_{1}$ is minimal, resulting in an almost flat ICE curve. However, as $x_{2}$ increases, the impact of $x_{1}$ becomes more significant, with the overall trend first increasing and then decreasing (which is due to the sinusoidal term). This effect is quantified and shown by the non-zero values of $\sigma_{I_{\mathrm{ice}}}$ and $\sigma_{\rho}$ discussed earlier. Further, it can be seen that the coloring seems to be disordered when the ICE curves for $x_{1}$ are color-coded according to $x_{4}$; which indicates that there is no interaction between how the change in $x_{4}$ affects $x_{1}$. 

\begin{figure}[hbt!]
	\centering
	\begin{subfigure}{.45\columnwidth}
		\includegraphics[width=1\columnwidth]{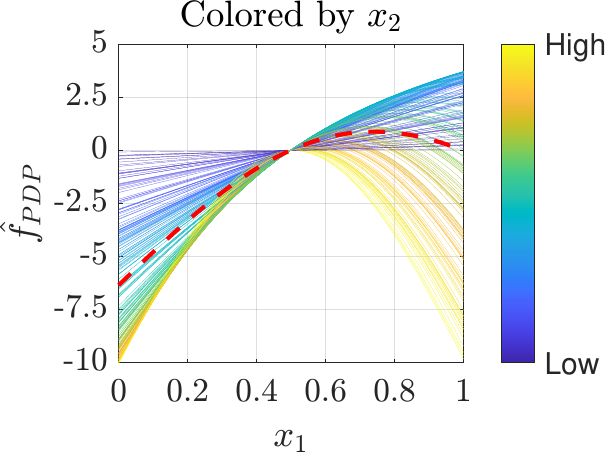}%
		\caption{PDP and ICE}%
		\label{fig:SHAP_plot_Vr_h_x1_x2}
	\end{subfigure}\hfill%
 	\begin{subfigure}{.45\columnwidth}
		\includegraphics[width=1\columnwidth]{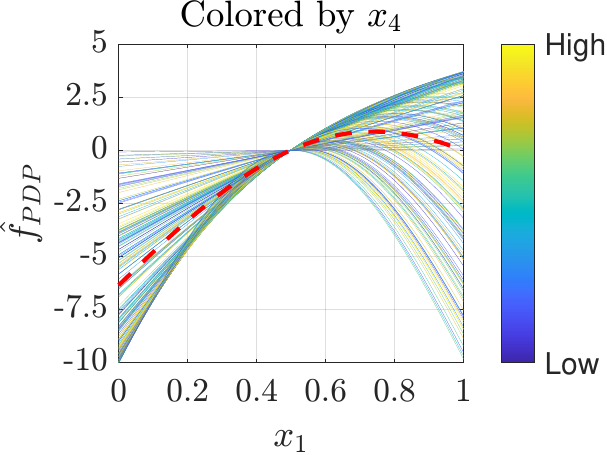}%
		\caption{PDP and ICE}%
		\label{fig:SHAP_plot_Vr_h_x1_x4}
	\end{subfigure}\hfill%
    \caption{PDP and ICE plots of $x_{1}$ for the 5-dimensional Friedman regression problem.}
	\label{fig:explainability_PDP_friedman_x1}
\end{figure}

The SHAP plots for $x_{1}$, while they do reveal interactions similar to the PDP/ICE plots, do not make it immediately clear that the effects of $x_{1}$ and $x_{2}$ are relatively flat for small values of both variables. This is because SHAP values emphasize the contribution of input variables to individual predictions, rather than illustrating the global effect of changing an input variable. As a result, while SHAP is valuable on its own, certain aspects may not be easily interpreted from it.

\begin{figure}[hbt!]
	\centering
 	\begin{subfigure}{.45\columnwidth}
		\includegraphics[width=1\columnwidth]{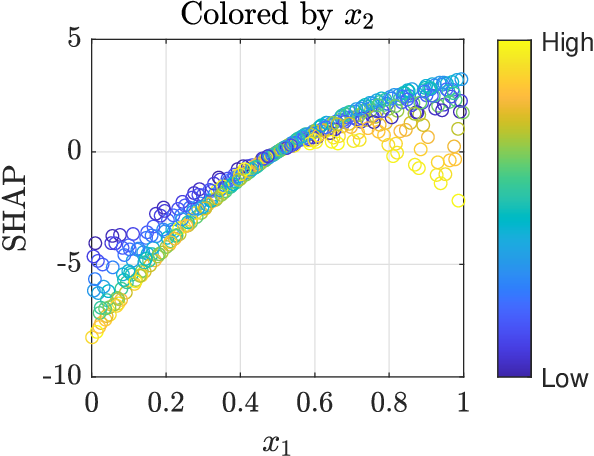}%
		\caption{SHAP}%
		\label{fig:SHAP_plot_Vr_m_x1_x2.eps}
	\end{subfigure}\hfill%
 	\begin{subfigure}{.45\columnwidth}
		\includegraphics[width=1\columnwidth]{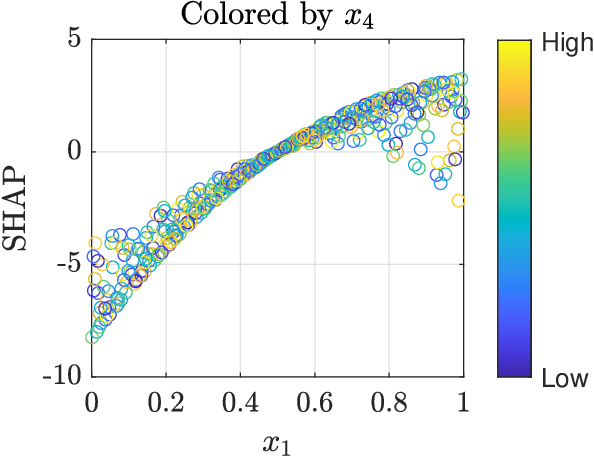}%
		\caption{SHAP}%
		\label{fig:SHAP_plot_Vr_m_x1_x4.eps}
	\end{subfigure}\hfill%
    \caption{SHAP dependence plots of $x_{1}$ for the 5-dimensional Friedman regression problem.}
	\label{fig:explainability_SHAP_friedman_x1}
\end{figure}

The joint PDP and SHAP interaction plots for $x_{1}-x_{2}$ are shown in Fig.~\ref{fig:explainability_friedman_int}, respectively (the other interactions are not shown since they are essentially non-existent). Note that the SHAP visualization was generated by holding the values of the other variables constant at the center of the input space. The color of those two plots corresponds to the value of partial dependence, for the left PDP interaction plot, and to the values of SHAP for the right interaction SHAP plot.
It is worth noting that both plots visualize two different aspects. That is, the joint PDP plot shows the average impact of changing $x_{1}$ and $x_{2}$ together on the Friedman function. Such a plot is primarily useful if the aim is to see both the main effects and the combined impact. Conversely, the SHAP interaction plot isolates the specific interaction between these two variables, allowing us to focus solely on the synergistic effect that cannot be attributed to either variable independently. 

\begin{figure}[hbt!]
	\centering
	\begin{subfigure}{.45\columnwidth}
		\includegraphics[width=1\columnwidth]{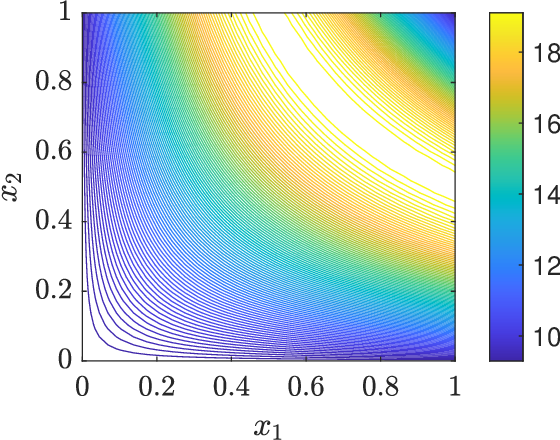}%
		\caption{Joint partial dependence function}%
		\label{fig:friedman_int_x1_x2}
	\end{subfigure}\hfill%
 	\begin{subfigure}{.45\columnwidth}
		\includegraphics[width=1\columnwidth]{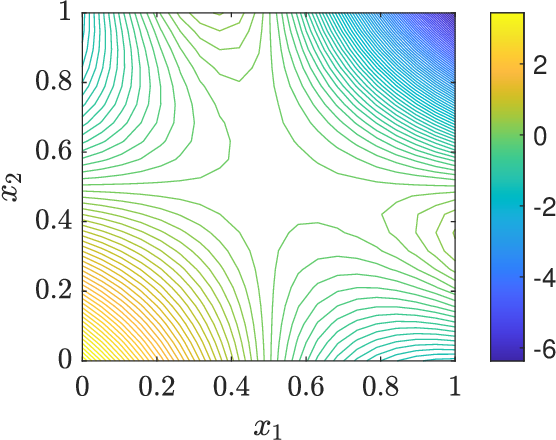}%
		\caption{Interaction SHAP}%
		\label{fig:friedman_int_x1_x4}
	\end{subfigure}\hfill%
    \caption{Two-way dependence plots of $x_{1}-x_{2}$ for the Friedman function according to the joint partial dependence function and interaction SHAP.}
	\label{fig:explainability_friedman_int}
\end{figure}

\subsection{5-variable wind turbine fatigue problem}
We now address the uncertainty analysis of a wind turbine fatigue problem, where the uncertainties involve five independent random variables related to wind conditions~\cite{graf2016high}. The input variables considered are wind speed ($V_{hub}$), wind direction ($\theta_{w}$), wave height ($H_{s}$), wave period ($T_{p}$), and wind-platform direction ($\theta_{p}$) with their histograms are shown in Fig.~\ref{fig:NREL_HIST}. Notice that the data was normalized within $[-1,1]^{5}$, as provided by the original paper. The side-side tower base bending moment ($M_{x,twr}$) is selected as the output of interest due to its nonlinear {behavior}, evaluated using FAST, NREL's aeroelastic code for simulating the coupled dynamic response of wind turbines~{\cite{jonkman2005fast}}. Subsequently, a PCE model with $p_{max}=5$ and data-driven orthogonal polynomials was built using 2000 samples, yielding $R^2=0.97$ on the validation data set comprising 500 samples, which is more than sufficient for knowledge discovery purposes. 
\begin{figure}[hbt!]
	\centering
		\includegraphics[width=1\columnwidth]{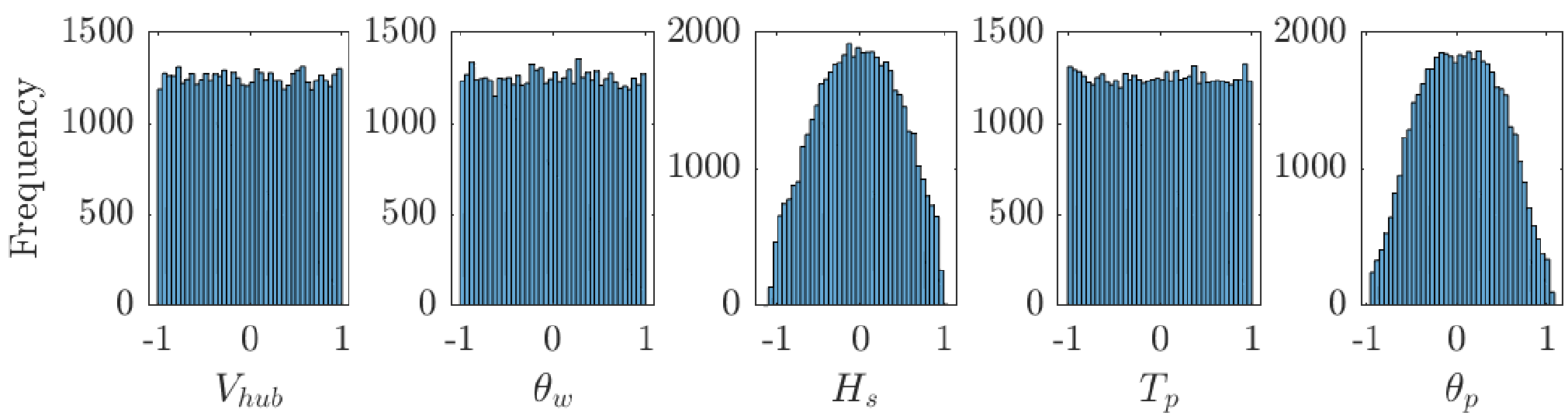}%
    \caption{Histograms of the input variables for the 5-variable wind turbine fatigue problem.}
	\label{fig:NREL_HIST}
\end{figure}
\subsubsection{Global Sensitivity Analysis}
Fig.~\ref{fig:NREL_GSA} shows the GSA results of three metrics  ($I_{pdp}$, $\mu_{I_{\mathrm{ice}}}$, and $\bar{Sh}$) for the 5-variable wind turbine fatigue problem (the {Sobol'} indices are not shown due to the different scale, but instead shown in Fig.~\ref{fig:two-way_GSA_NREL}). All GSA metrics agree that the most important variable is $V_{hub}$ followed by $\theta_{w}$ and $H_{s}$.  On the other hand, $T_{p}$ and $\theta_{p}$ are deemed to be insignificant. The averaged SHAP perceives the importance of $\theta_{w}$ is about half that of $H_{s}$, with $I_{pdp}$ and $\mu_{I_{\mathrm{ice}}}$ perceiving that the former is slightly more important than the latter. It is interesting to see from the barplot that, although $V_{hub}$ is the most impactful variable, the effect of interaction is felt the strongest by $\theta_{w}$ and $H_{s}$ as indicated by the higher value of $\sigma_{I_{\mathrm{ice}}}$ compared to $V_{hub}$. The cause for this will be discussed soon by visualizing the PDP and ICE curves. It is worth noting again that $\sigma_{I_{\mathrm{ice}}}$ does not indicate which combinations of variables contribute the strongest to the interaction; however, the information from $\sigma_{I_{\mathrm{ice}}}$ is useful in revealing the variability of the impact, complemented by the average impact from $\mu_{I_{\mathrm{ice}}}$. Lastly, the impact of interactions on $T_{p}$ and $\theta_{p}$ is nevertheless small, and their importance can be neglected. 

The violin plots of {ICE-based correlation values} are shown in Fig.~\ref{fig:NREL_COR}. The actual $\sigma_{\rho}$ values for $V_{hub}$, $\theta_{w}$, $H_{s}$, $T_{p}$, and $\theta_{p}$ are  $0.058$,    $0.275$,  $0.160$,    $0.444$, and  $0.573$, respectively.  These violin plots suggest that the interaction does not significantly alter the main trend of how $V_{hub}$ influences the moment. However, a notable shift in trends is evident for $\theta_{w}$ and $H_{s}$, indicating that the marginal behavior of these two variables is strongly influenced by their interactions with each other or with other variables. While the shifts in trend due to $T_{p}$ and $\theta_{p}$ appear significant, it is important to note that these changes may be the result of minor variations, as the impact of these two variables is relatively insignificant. Finally, quantification of the impact of interactions is shown in Fig.~\ref{fig:two-way_GSA_NREL}, with the diagonal entries showing the main effect (for {Sobol'} indices, the first-order {Sobol'} indices). All GSA agree that the most important interaction is between $H_{s}$ and $\theta_{w}$, followed by $V_{hub}-H_{s}$. Notice that the off-diagonal colors representing the correlations between $V_{hub}$, $\theta_{w}$, and $H_{s}$ are brighter for averaged SHAP and $I_{pdp}$ compared to {Sobol'} indices. This difference arises because {Sobol'} indices are variance-based, which amplifies the importance of the most influential variables. 

\begin{figure}[hbt!]
	\centering
 	\begin{subfigure}{.40\columnwidth}
		\includegraphics[width=1\columnwidth]{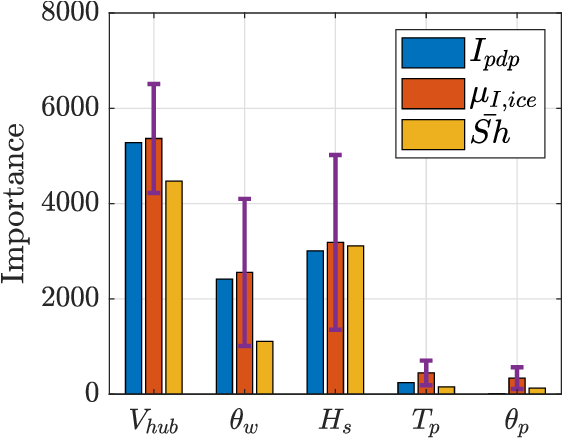}%
		\caption{Barplot of $I_{pdp}$, $\mu_{I_{\mathrm{ice}}}$, and $\bar{Sh}$}%
		\label{fig:NREL_GSA}
	\end{subfigure}\hfill%
  	\begin{subfigure}{.40\columnwidth}
		\includegraphics[width=1\columnwidth]{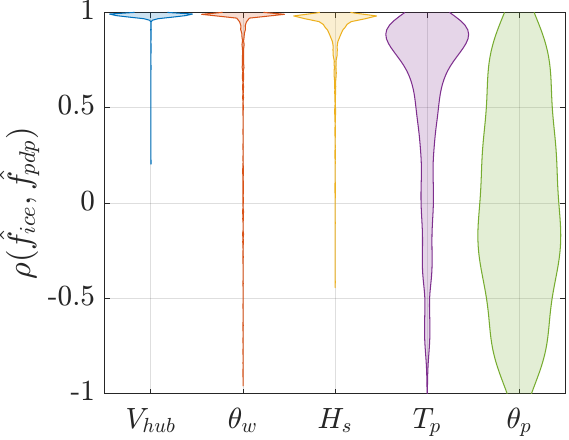}%
      \caption{\textcolor{black}{Violin plot of {ICE-based correlation values}}.}
		\label{fig:NREL_COR}
	\end{subfigure}\hfill%
    \caption{GSA results for the 5-variable wind turbine fatigue problem ($I_{pdp}$, $\mu_{I_{\mathrm{ice}}}$, and $\bar{Sh}$) and the ICE correlation values.}
	\label{fig:GSA_COR_NREL}
\end{figure}

\begin{figure}[hbt!]
	\centering
 	\begin{subfigure}{.33\columnwidth}
		\includegraphics[width=1\columnwidth]{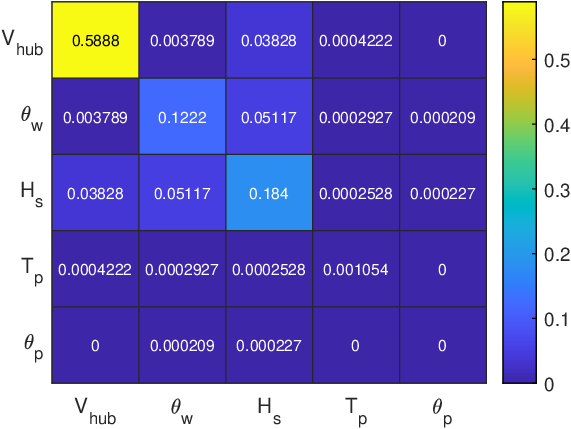}%
		\caption{{Sobol'} indices}%
		\label{fig:NREL_2WAY_SOB}
	\end{subfigure}\hfill%
 	\begin{subfigure}{.33\columnwidth}
		\includegraphics[width=1\columnwidth]{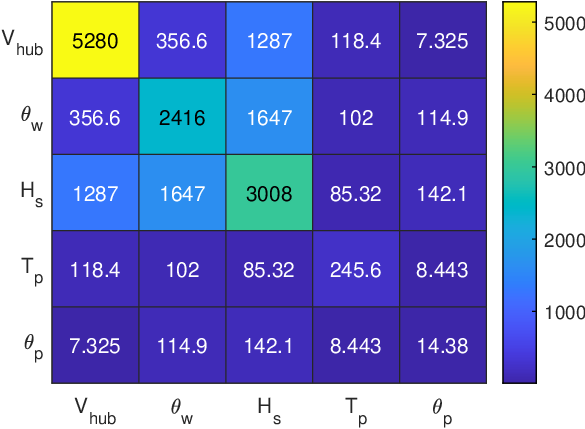}%
		\caption{$I_{pdp}$}%
		\label{fig:NREL_2WAY_IPDP}
	\end{subfigure}\hfill%
 	\begin{subfigure}{.33\columnwidth}
		\includegraphics[width=1\columnwidth]{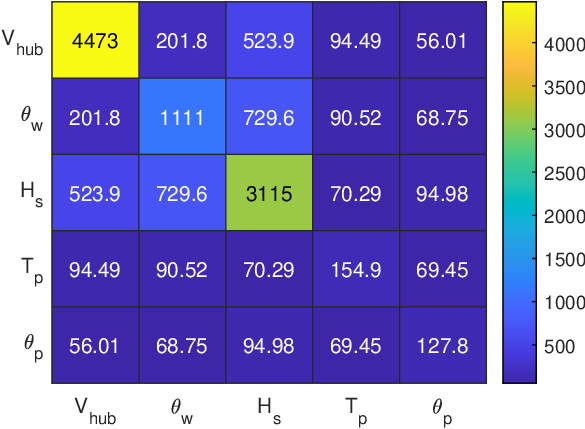}%
		\caption{SHAP}%
		\label{fig:NREL_2WAY_SHAP}
	\end{subfigure}\hfill%
    \caption{Two-way GSA results for the 5-variable wind turbine fatigue problem. The diagonal entries show the main effect importance, while the off-diagonal entries show the two-way interaction importance.}
	\label{fig:two-way_GSA_NREL}
\end{figure}

\subsubsection{Visualization}
Fig.~\ref{fig:NREL_PDP} shows the PDPs and ICE curves for the three most significant variables, clearly showing how they affect the model's behavior. The effect of $V_{hub}$ is nonlinear, with $M_{x,twr}$ generally increasing as $V_{hub}$ rises, although there is a slight decrease when $V_{hub}$ slightly exceeds its nominal value. Similarly, an increase in $\theta_{w}$ tends to raise $M_{x,twr}$. However, the actual trend influenced by $H_{s}$ is more complex, as revealed by the ICE plots and two-way PDPs. The {harmonic looking} pattern in $M_{x,twr}$ due to changes in $H_{s}$ is evident, with peaks observed around $H_{s}=-0.7$ and $H_{s}=0.7$. Additionally, as can be seen from the ICE plots, the impact on the bending moment is more pronounced when $\theta_{w}$ is large. Linking this observation with the GSA metrics, the large range of $\hat{f}_{pdp}$ of $V_{hub}$ contributes to its large importance. However, the impact of interaction itself is relatively small, as can be seen from the smaller deviation of ICE curves from the PDP. On the other hand, the deviation of ICE curves from their partial dependence function can be clearly seen for $\theta_{w}$ and $H_{s}$, particularly the latter; Overall, this contributes to their relatively large values of $\sigma_{I_{\mathrm{ice}}}$. For comparison, Fig.~\ref{fig:NREL_SHAP} shows the SHAP dependence plots for the same set of variables. While the insights from SHAP and PDP/ICE are similar, the former appears more disordered, which is understandable given that SHAP values exist for every point in the input space. In that sense, it is easier to understand the effect of changing $V_{hub}$, $\theta_{w}$, and $H_{s}$ since PDP and ICE curves are smooth.

\begin{figure}[hbt!]
	\centering
	\begin{subfigure}{.32\columnwidth}
		\includegraphics[width=1\columnwidth]{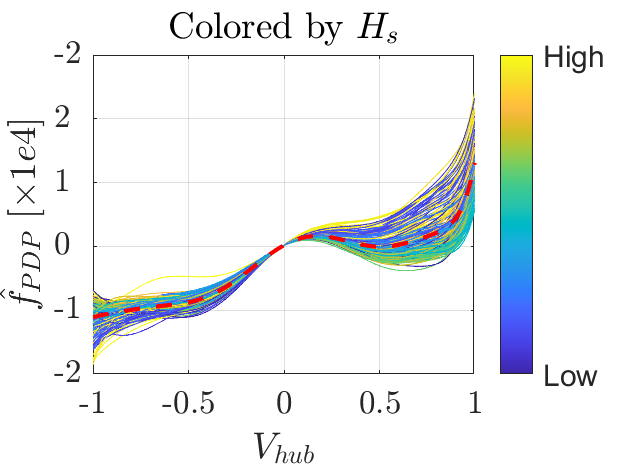}%
      \caption{$V_{hub}$}
		\label{fig:NREL_PDP_X1}
    \end{subfigure}
	\begin{subfigure}{.32\columnwidth}
		\includegraphics[width=1\columnwidth]{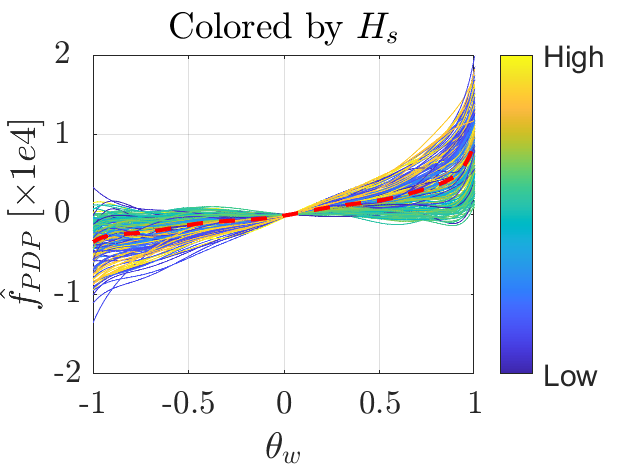}%
      \caption{$\theta_{w}$}
		\label{fig:NREL_PDP_X2}
    \end{subfigure}
	\begin{subfigure}{.32\columnwidth}
		\includegraphics[width=1\columnwidth]{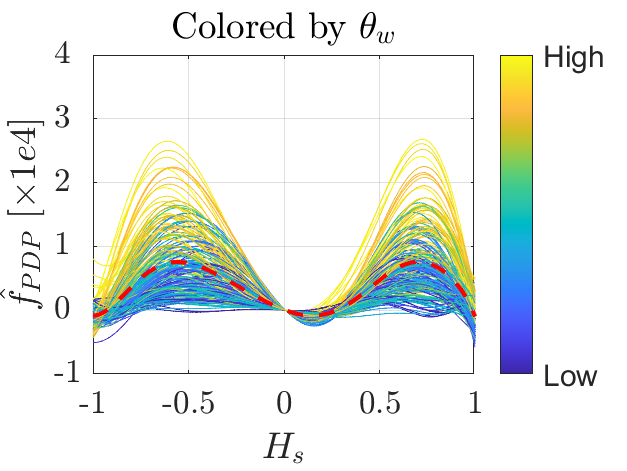}%
      \caption{$H_{s}$}
		\label{fig:NREL_PDP_X3}
    \end{subfigure}
        \caption{PDP and ICE plots for the first three important variables.}
	\label{fig:NREL_PDP}
\end{figure}

\begin{figure}[hbt!]
	\centering
	\begin{subfigure}{.32\columnwidth}
		\includegraphics[width=1\columnwidth]{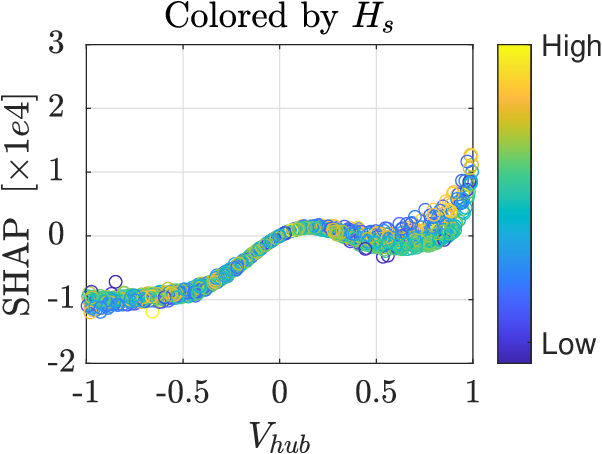}%
      \caption{$V_{hub}$}
		\label{fig:NREL_SHAP_X1}
    \end{subfigure}
	\begin{subfigure}{.32\columnwidth}
		\includegraphics[width=1\columnwidth]{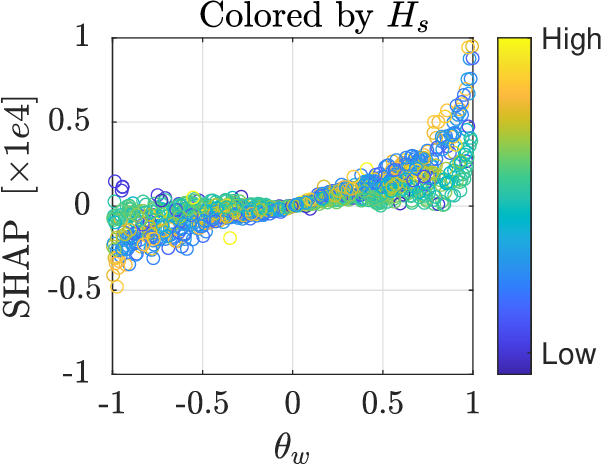}%
      \caption{$\theta_{w}$}
		\label{fig:NREL_SHAP_X2}
    \end{subfigure}
	\begin{subfigure}{.32\columnwidth}
		\includegraphics[width=1\columnwidth]{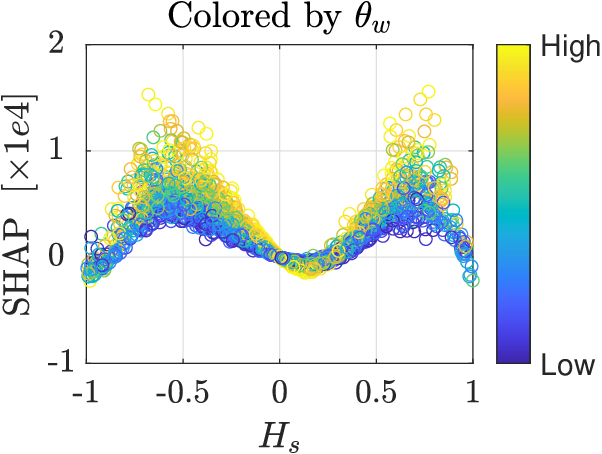}%
      \caption{$H_{s}$}
		\label{fig:NREL_SHAP_X3}
    \end{subfigure}
        \caption{SHAP dependence plots for the first three important variables.}
	\label{fig:NREL_SHAP}
\end{figure}

Fig.~\ref{fig:NREL_PDP_int} depicts the two-way PDP plots for the combinations of the three most important variables ({\textit{i.e.}}, $V_{hub}-\theta_{w}$, $V_{hub}-H_{s}$, and $\theta_{w}-H_{s}$). The complex interactions between $\theta_{w}$ and $H_{s}$ are clearer when plotting their joint partial dependence function. For example, it can be observed that the marginal impact of $H_{s}$ becomes larger when the corresponding $\theta_{w}$ is large. To ensure the reliability of wind turbines, the GSA metrics and plots indicate that specific conditions can result in significant bending moments. As previously mentioned, these conditions occur when $\theta_{w}$ and $V_{hub}$ are large, and $H_{s}$ is approximately $-0.7$ or $0.7$. One limitation of the joint partial dependence plot is that it can be challenging to discern the interaction strength between $V_{hub}$ and $\theta_{w}$ (which is actually small). Therefore, this information should be considered alongside the analysis of GSA metrics for a more complete understanding. 
\begin{figure}[hbt!]
	\centering
	\begin{subfigure}{.32\columnwidth}
		\includegraphics[width=1\columnwidth]{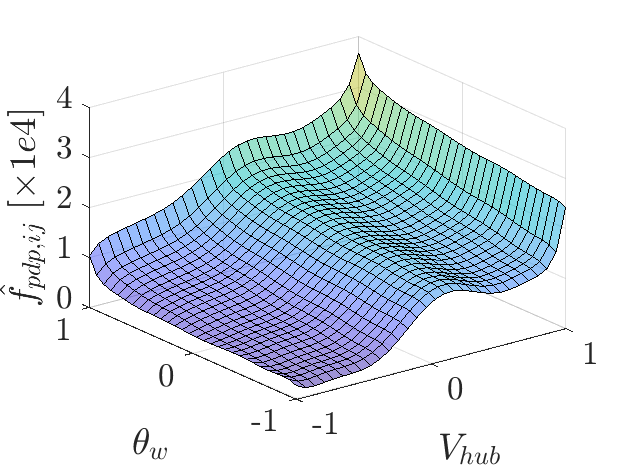}%
      \caption{$V_{hub}-\theta_{w}$}
		\label{fig:NREL_PDP_X1_X2}
    \end{subfigure}
	\begin{subfigure}{.32\columnwidth}
		\includegraphics[width=1\columnwidth]{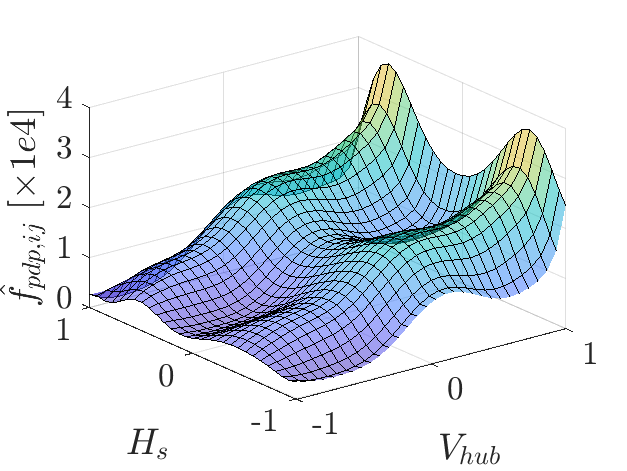}%
      \caption{$V_{hub}-H_{s}$}
		\label{fig:NREL_PDP_X1_X3}
    \end{subfigure}
	\begin{subfigure}{.32\columnwidth}
		\includegraphics[width=1\columnwidth]{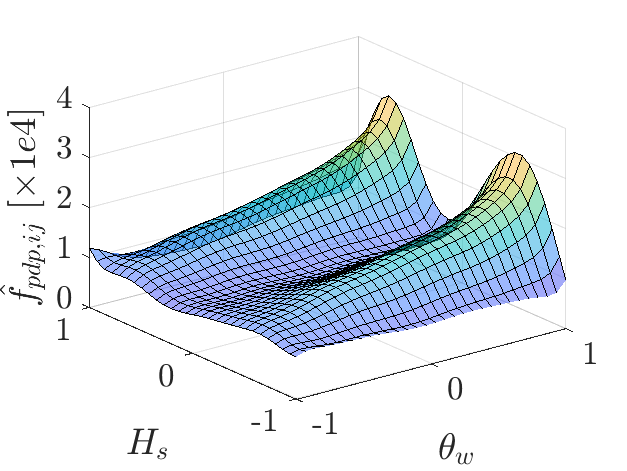}%
      \caption{$\theta_{w}-H_{s}$}
		\label{fig:NREL_PDP_X2_X3}
    \end{subfigure}
        \caption{Two-way PDP plots for the first three important variables.}
	\label{fig:NREL_PDP_int}
\end{figure}

\subsection{9-variable Airfoil database problem}
\subsubsection{Data set}
We evaluated the usefulness of PDP and ICE, particularly the global sensitivity metrics derived from ICE, for knowledge discovery in an airfoil database from Agarwal et al.~\cite{agarwal2024comprehensive}, where aerodynamic coefficients were computed using OpenFOAM at $Re = 10^{5}$. The dataset was built by fitting eight CST parameters~\cite{kulfan2008universal} to airfoils from the UIUC Airfoil Data Site. Aerodynamic performance was then obtained using the simpleFOAM solver with the viscous Spalart–Allmaras turbulence model. Simulations were conducted at 13 different angles of attack ($\alpha$), spanning from -4 to 8 degrees, with an increment of 1 degree. The output of interest is the drag coefficient ($C_{d}$). Four basis functions are deemed sufficient for this purpose (Fig.~\ref{fig:CST_basis_func}). Hence, in total there are nine input variables, collected into a vector $\boldsymbol{x}=[\alpha, A_{l,1},A_{l,2},A_{l,3},A_{l,4},A_{u,1},A_{u,2},A_{u,3},A_{u,4}]^{T}$, where $A$ is the CST coefficients and the subscript $l$ and $u$ indicate upper and lower surface, respectively.

After removing erroneous aerodynamic data, particularly abnormal \(C_d\)–\(\alpha\) trends, the dataset comprised 6{,}600 samples from 510 airfoils at 13 angles of attack. Figure~\ref{fig:CST_airfoil_shapes} shows 100 representative airfoils generated around the center of the CST coefficient range. A data-driven PCE model (\(p_{\max}=4\), \(q=1\)) was trained on 6{,}000 samples and tested on 600, achieving \(R^{2}=0.994\), demonstrating accuracy suitable for knowledge discovery.

\begin{table}[ht!]
\centering
\begin{tabular}{|c|c|c|}
\hline
\textbf{Variable} & \textbf{Lower Bound} & \textbf{Upper Bound} \\ \hline
$\alpha$ [deg]          & -4.000          & 8.000              \\ \hline
$A_{l,1}$         & -0.236              & -0.077              \\ \hline
$A_{l,2}$         & -0.249              & 0.090               \\ \hline
$A_{l,3}$         & -0.379              & 0.068               \\ \hline
$A_{l,4}$         & -0.219              & 0.237               \\ \hline
$A_{u,1}$         & 0.131               & 0.338               \\ \hline
$A_{u,2}$         & 0.104               & 0.475               \\ \hline
$A_{u,3}$         & 0.081               & 0.417               \\ \hline
$A_{u,4}$         & 0.041               & 0.429               \\ \hline
\end{tabular}
\caption{Variables and their corresponding lower and upper bounds for the 9-variable airfoil problem.}
	\label{tbl:CST_airfoil}
\end{table}

\begin{figure}[hbt!]
	\centering
	\begin{subfigure}{.45\columnwidth}
		\includegraphics[width=1\columnwidth]{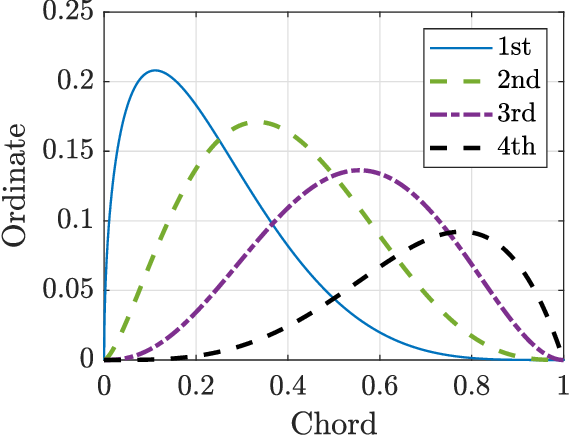}%
      \caption{Basis functions}
		\label{fig:CST_basis_func}
    \end{subfigure}
	\begin{subfigure}{.45\columnwidth}
		\includegraphics[width=1\columnwidth]{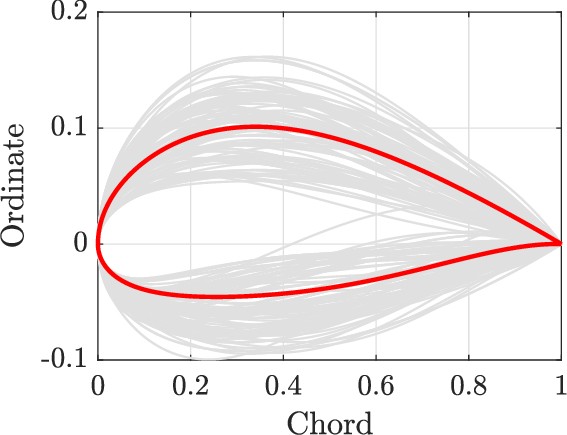}%
      \caption{CST-generated airfoils}
		\label{fig:CST_airfoil_shapes}
    \end{subfigure}
        \caption{Four basis functions of CST and a snapshot of airfoils from the database with the airfoil at the centre of the CST range.}
	\label{fig:CST_airfoil}
\end{figure}

\subsubsection{Result}
The GSA and ICE correlation results for the 9-variable airfoil problem are shown in Fig.~\ref{fig:airfoil_GSA_plot}. The two-way GSA results are shown in Fig.~\ref{fig:two-way_GSA_AIRFOIL} (refer to the off-diagonal entries). 
 Fig.~\ref{fig:PDP_ICE_airfoil_database} shows the one-dimensional PDP and ICE plots for select variables. As can be seen in Fig.~\ref{fig:airfoil_GSA_plot}, all GSA metrics agreed on ranking the importance of the four first important input variables ({\textit{i.e.}}, $\alpha$, $A_{u,2}$, $A_{u,3}$, and $A_{u,1}$). From GSA, it can be seen that $\alpha$ is the most important parameter in terms of drag generation, which makes sense since changing $\alpha$ significantly changes the pressure and shear-stress distribution at the airfoil; However, our main focus here is on the CST parameters. All GSA metrics consistently indicate that the upper surface CST parameters dominate drag generation, reflecting their influence on boundary-layer separation. The most critical variable is \( A_{u,2} \), which shapes the upper surface near the maximum thickness. Minor discrepancies in the ranking of less influential variables (lower-surface terms and \( A_{u,4} \)) are negligible.

 % In that sense, all GSA metrics agree that the upper surface CST parameters are more important than the lower surface in drag generation. This makes sense from the viewpoint of aerodynamics due to its important role in affecting boundary layer separation. It can also be seen that the most important CST variable is $A_{u,2}$, which primarily controls the change of the upper surface close to the location of the maximum thickness. Although there are differences in the exact ranking of the least important variables (including all lower surface variables and $A_{u,4}$), these discrepancies are not significant because these variables are less important compared to the others.

It can also be seen that the values of $\mu_{I_{\mathrm{ice}}}$ are larger than $I_{pdp}$. In this regard, such differences are notably relatively significant for the first three lower surface CST variables, indicating that the respective $I_{pdp}$ values are smaller than $\mu_{I_{\mathrm{ice}}}$ due to alternating positive and negative correlations. It can also be seen that the impact of interactions between variables on the ICE curves is strong, as evidenced by the significantly high value of $\sigma_{I_{\mathrm{ice}}}$ for all variables (see the error bar in Fig.~\ref{fig:airfoil_GSA_plot}). Therefore, the drag response surface is characterized by nonlinearity and strong interactions between input variables. It is worth noting that the values of $\sigma_{I_{\mathrm{ice}}}$ themselves do not directly inform which variables contribute to the interactions for a specific input variable. Although it makes sense that the interaction is primarily due to $\alpha$ since it is the most important input variable, the interaction with other CST variables also contributes to such differences. 

\begin{figure}[hbt!]
	\centering
	\begin{subfigure}{.45\columnwidth}
		\includegraphics[width=1\columnwidth]{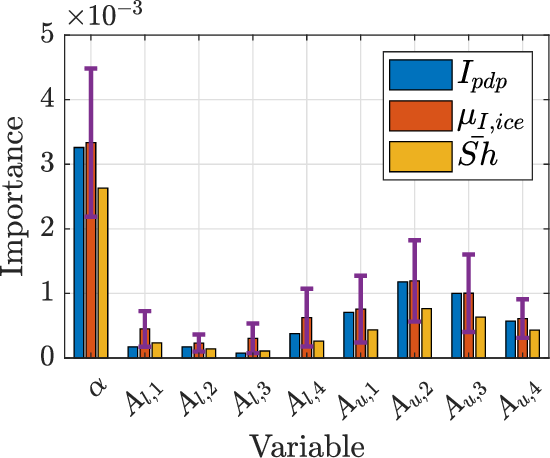}%
		\caption{Barplot of $I_{pdp}$, $\mu_{I_{\mathrm{ice}}}$, and $\bar{Sh}$}%
		\label{fig:airfoil_GSA}
    \end{subfigure}
	\begin{subfigure}{.46\columnwidth}
		\includegraphics[width=1\columnwidth]{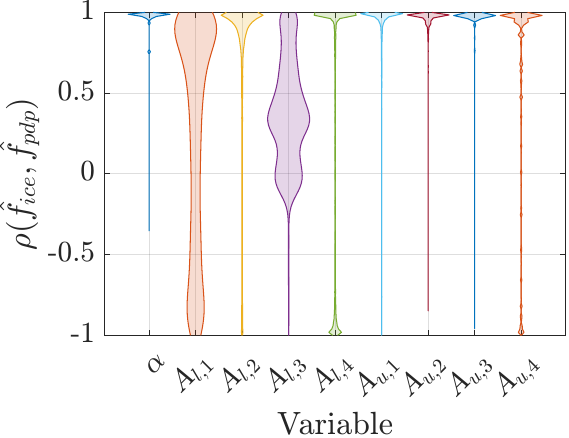}%
       \caption{\textcolor{black}{Violin plot of {ICE-based correlation values}}.}
		\label{fig:airfoil_COR}
    \end{subfigure}
        \caption{GSA and ICE correlation values results for the 9-variable airfoil problem}
	\label{fig:airfoil_GSA_plot}
\end{figure}

\begin{table}[ht!]
\centering
\begin{tabular}{|c|c|}
\hline
\textbf{Variable} & $\sigma_{\rho}$ \\ \hline
$\alpha$    & 0.011 \\ \hline
$A_{l,1}$         & 0.705 \\ \hline
$A_{l,2}$         & 0.491 \\ \hline
$A_{l,3}$         & 0.309 \\ \hline
$A_{l,4}$         & 0.829 \\ \hline
$A_{u,1}$         & 0.319 \\ \hline
$A_{u,2}$         & 0.050 \\ \hline
$A_{u,3}$         & 0.012 \\ \hline
$A_{u,4}$         & 0.400 \\ \hline
\end{tabular}
\caption{Standard deviation of ICE correlation values for the 9-variable airfoil problem.}
\label{tbl:CST_airfoil_ICE_cor}
\end{table}

The two-way GSA results are shown in the off-diagonal entries of Fig.~\ref{fig:two-way_GSA_AIRFOIL}. We focus exclusively on the CST coefficients, omitting the impact of the angle of attack to enhance plot clarity. Compared to the main effects, as expected, the interaction effects are relatively weaker. As discussed, the strongest interaction occurs between $A_{u,2}$ and $A_{u,3}$, which is expected since both are the most influential variables. However, it is important to note that these interactions alter the relationships between the input variables and the drag coefficient. 
\begin{figure}[hbt!]
	\centering
 	\begin{subfigure}{.33\columnwidth}
		\includegraphics[width=1\columnwidth]{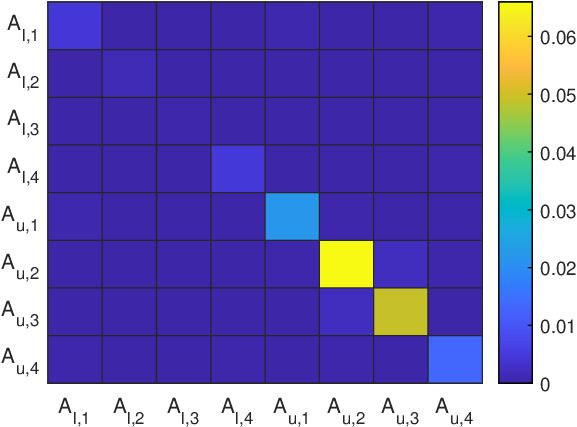}%
		\caption{{Sobol'} indices}%
		\label{fig:AIRFOIL_2WAY_Sobol'}
	\end{subfigure}\hfill%
 	\begin{subfigure}{.33\columnwidth}
		\includegraphics[width=1\columnwidth]{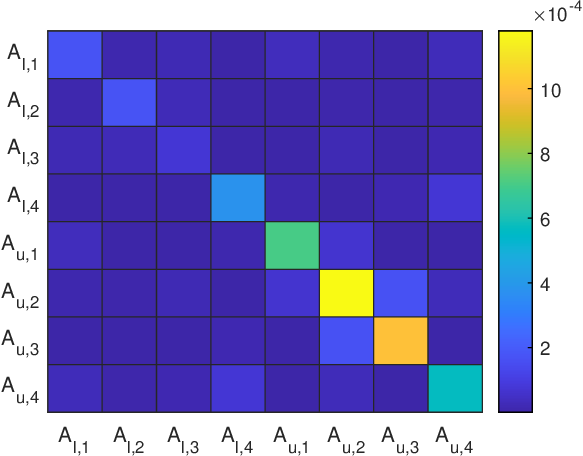}%
		\caption{$I_{pdp}$}%
		\label{fig:AIRFOIL_2WAY_I_PDP}
	\end{subfigure}\hfill%
 	\begin{subfigure}{.33\columnwidth}
		\includegraphics[width=1\columnwidth]{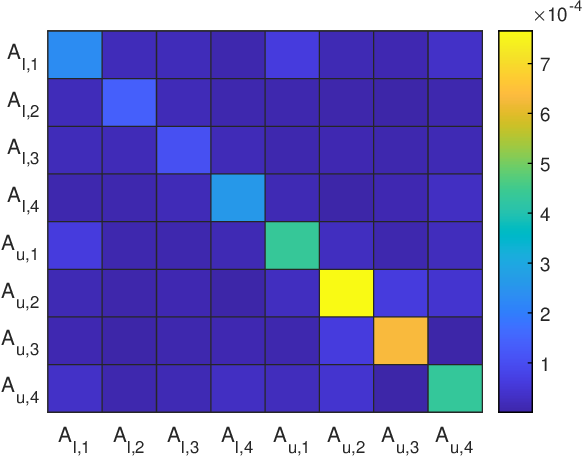}%
		\caption{SHAP}%
		\label{fig:AIRFOIL_2WAY_SHAP.eps}
	\end{subfigure}\hfill%
    \caption{Two-way GSA results for the 9-variable airfoil database problem. The diagonal entries show the main effect importance, while the off-diagonal entries show the two-way interaction importance.}
	\label{fig:two-way_GSA_AIRFOIL}
\end{figure}

The violin plot of ICE correlation values presented in Fig.~\ref{fig:airfoil_COR} offers several key insights. The standard deviations of ICE correlation values are shown in Table~\ref{tbl:CST_airfoil_ICE_cor}. First, it is evident that lower surface CST parameters, particularly $A_{l,1}$ and $A_{l,4}$, exhibit high dispersion in terms of ICE correlation values, indicating significant variations in their ICE curves due to interactions with other variables (as confirmed in Fig.~\ref{fig:PDP_ICE_airfoil_database}). Conversely, the lower dispersion for upper surface CST parameters suggests that interactions have a minimal effect on the overall trend, though the magnitude of the impact may still vary, as seen in the corresponding PDP and ICE plots. Overall, the lower surface CST parameters are more prone to trend changes due to interactions. In contrast, the upper surface parameters show less variation in trend but remain the most important variables in terms of significance. For the PDP and ICE curves of all variables, the trend consistently remains monotonic, reinforcing the hypothesis that the drag response of an airfoil is an unimodal function~\cite{bons2019multimodality}, even in the presence of nonlinearity.

The impact of $A_{l,1}$ is characterized by complex interaction with $A_{u,1}$ (see Fig.~\ref{fig:airfoil_PDP_ICE_AL1}). In this sense, a nonlinear trend that also changes in terms of correlation with the value of $A_{l,1}$ is present. Such an alternating trend results in a partial dependence function close to the flat line, which is why its $I_{pdp}$ is notably smaller than $\mu_{I_{\mathrm{ice}}}$. The same trend also exists for $A_{l,4}$ as shown in Fig.~\ref{fig:airfoil_PDP_ICE_AL4}. We decided to plot how $A_{l,4}$ interacts with $\alpha$ for demonstration purposes. In this regard, we can see that the association of the PDP of $A_{l,4}$ with its input value changes depending on the angle of attack. To be more exact, increasing $A_{l,4}$ results in higher drag for a positive angle of attack, while the opposite trend occurs for a smaller angle of attack. The PDP and ICE curves for \(A_{u,2}\) clearly demonstrate its interaction with \(A_{u,3}\) in influencing drag (see Fig.~\ref{fig:airfoil_PDP_ICE_AU2}). Notably, the effect of \(A_{u,2}\) is most pronounced when \(A_{u,3}\) is at a high value. Conversely, this implies that altering \(A_{u,2}\) has minimal impact on drag when \(A_{u,3}\) is small. Aerodynamically, the combinations of small $A_{u,2}$ and $A_{u,3}$ correspond to airfoils with reduced thickness, resulting in lower drag. 

The SHAP dependence plots roughly reveal the same trend as the PDP and ICE plots (see Fig.~\ref{fig:SHAP_airfoil_database}). However, the colouring of SHAP plots might be too crowded, especially if the dots are too scattered and complex interactions exist. For example, it is quite difficult to decipher the trend in the $A_{u,2}$ plot regarding how it interacts with $A_{u,3}$. In this regard, PDP and ICE provide more straightforward and easier-to-read plots than SHAP.

\begin{figure}[hbt!]
	\centering
  	\begin{subfigure}{.33\columnwidth}
		\includegraphics[width=1\columnwidth]{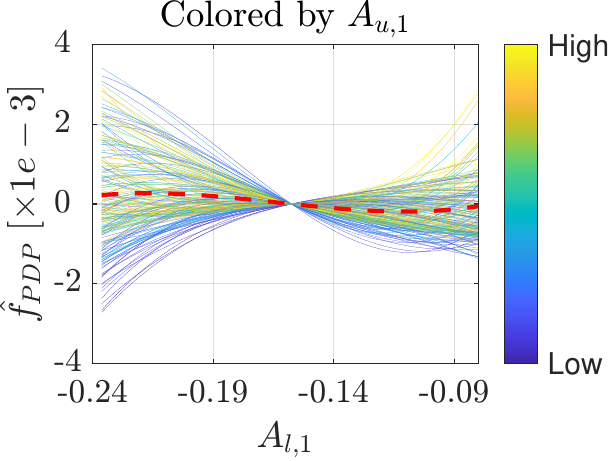}%
		\caption{$A_{l,1}$}%
		\label{fig:airfoil_PDP_ICE_AL1}
	\end{subfigure}\hfill%
  	\begin{subfigure}{.33\columnwidth}
		\includegraphics[width=1\columnwidth]{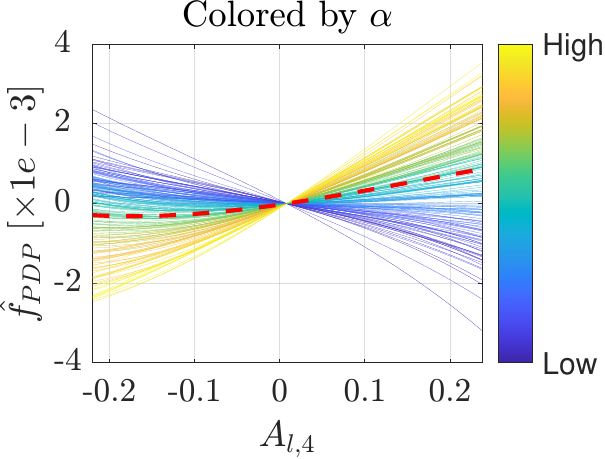}%
	\caption{$A_{l,4}$}%
		\label{fig:airfoil_PDP_ICE_AL4}
	\end{subfigure}\hfill%
   	\begin{subfigure}{.33\columnwidth}
		\includegraphics[width=1\columnwidth]{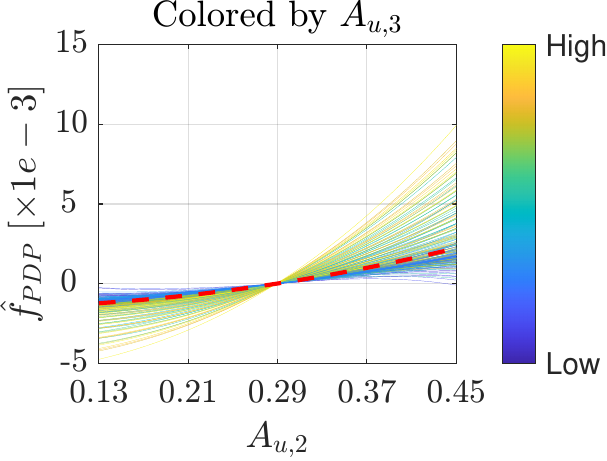}%
	\caption{$A_{u,2}$}%
		\label{fig:airfoil_PDP_ICE_AU2}
	\end{subfigure}\hfill%
    \caption{PDP and ICE plots for selected inputs for the 9-variable airfoil database problem}
	\label{fig:PDP_ICE_airfoil_database}
\end{figure}

\begin{figure}[hbt!]
	\centering
 	\begin{subfigure}{.33\columnwidth}
		\includegraphics[width=1\columnwidth]{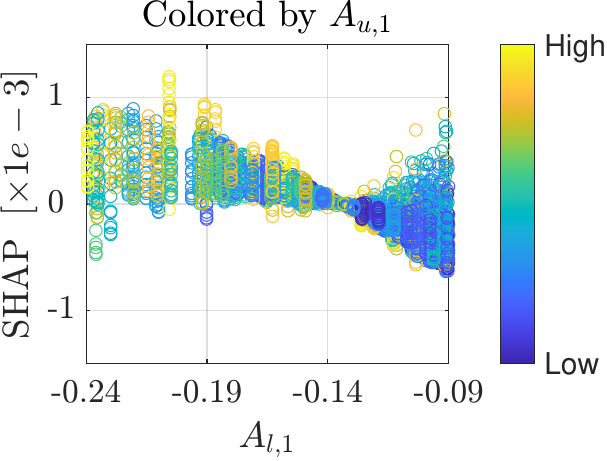}%
		\caption{$\alpha$}%
		\label{fig:airfoil_SHAP_AL1.eps}
	\end{subfigure}\hfill%
 	\begin{subfigure}{.33\columnwidth}
		\includegraphics[width=1\columnwidth]{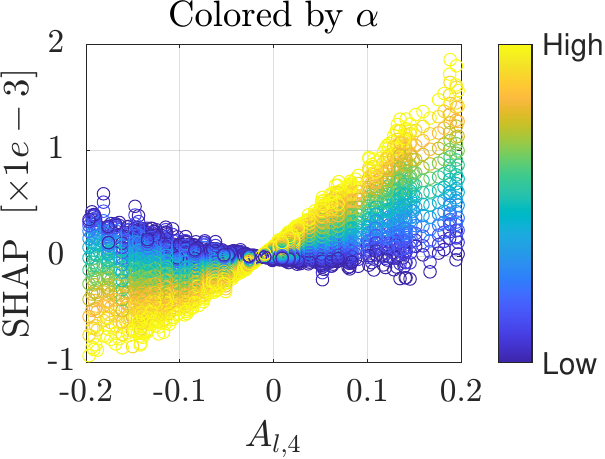}%
		\caption{$\alpha$}%
		\label{fig:airfoil_SHAP_AL4.eps}
	\end{subfigure}\hfill%
 	\begin{subfigure}{.33\columnwidth}
		\includegraphics[width=1\columnwidth]{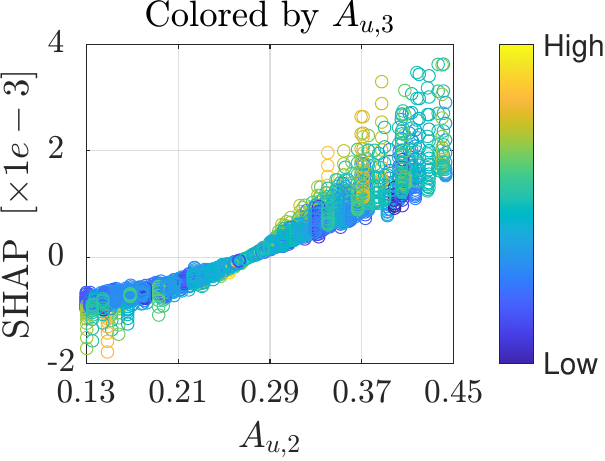}%
		\caption{$\alpha$}%
		\label{fig:airfoil_SHAP_AU2.eps}
	\end{subfigure}\hfill%
    \caption{SHAP dependence plots for selected inputs for the 9-variable airfoil database problem}
	\label{fig:SHAP_airfoil_database}
\end{figure}

Finally, Fig.~\ref{fig:joint_PDP_airfoil_database} visualizes the joint dependence plot for select variables, namely: $A_{l,1}-A_{u,1}$, $A_{l,4}-\alpha$, and $A_{u,2}-A_{u,3}$. The trend of $A_{u,2}-A_{u,3}$, in which it can be seen that minimum drag is achieved by decreasing the two variables together ({\textit{i.e.}}, due to the reduced thickness). On the other hand, the trend is more complex and wavy for the joint plot for $A_{u,1}$ and $A_{l,1}$. The change in the association of $A_{l,1}$ to drag as a function of $A_{u,1}$ can also be seen from this plot.

\begin{figure}[hbt!]
	\centering
   	\begin{subfigure}{.33\columnwidth}
		\includegraphics[width=1\columnwidth]{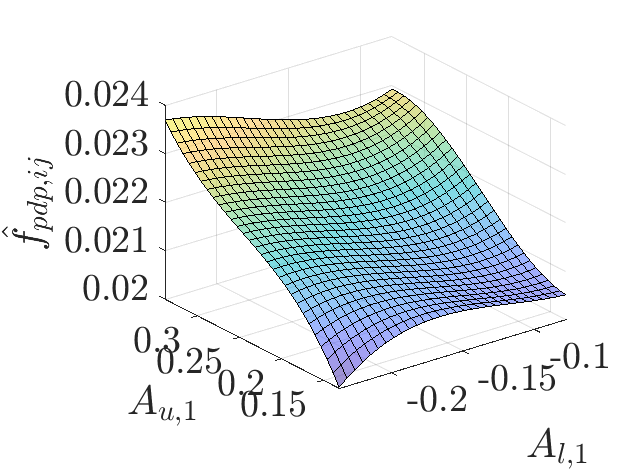}%
		\caption{$A_{l,1}-A_{u,1}$}%
		\label{fig:airfoil_joint_PDP_AU1_AL1_surf}
	\end{subfigure}\hfill%
 	\begin{subfigure}{.33\columnwidth}
		\includegraphics[width=1\columnwidth]{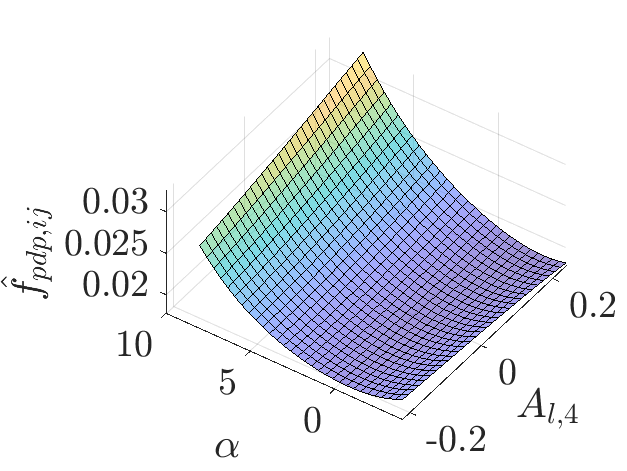}%
		\caption{$A_{l,4}-\alpha$}%
		\label{fig:airfoil_joint_PDP_AL4_alpha_surf}
	\end{subfigure}\hfill%
  	\begin{subfigure}{.33\columnwidth}
		\includegraphics[width=1\columnwidth]{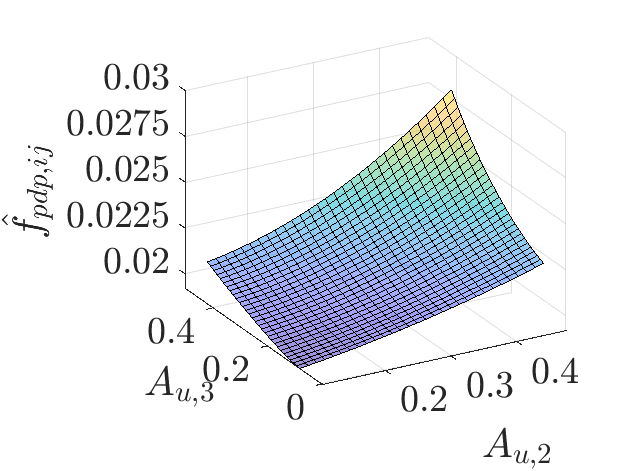}
		\caption{$A_{u,2}-A_{u,3}$}%
		\label{fig:airfoil_joint_PDP_AU2_AU3_surf}
	\end{subfigure}\hfill%
    \caption{PDP and ICE plots for selected inputs for the 9-variable airfoil database problem}
	\label{fig:joint_PDP_airfoil_database}
\end{figure}

\subsection{Interim summary}

In the workflow of knowledge discovery in engineering, our ICE-based summaries \(\mu_{I_{\mathrm{ice}},x_j}\), \(\sigma_{I_{\mathrm{ice}},x_j}\), and $\sigma_{\rho}$ are useful complements, since they capture instance-level spread and heterogeneity that scalar variance indices aggregate away. Based on the experiments that we performed, we summarize our main findings and key strengths of the proposed method as follows: 
\begin{enumerate}
    \item The ICE-based feature importance  metrics ($\mu_{I_{\mathrm{ice}}}$ and $\sigma_{I_{\mathrm{ice}}}$) provide richer insights than the traditional PDP-based metric ($I_{pdp}$). While $\mu_{I_{\mathrm{ice}}}$ captures effects missed by $I_{pdp}$, $\sigma_{I_{\mathrm{ice}}}$ \textcolor{black}{quantifies the heterogeneity of this influence across different conditional settings, thereby reflecting the strength of interaction effects. When used in conjunction with SHAP values and Sobol' indices, the proposed ICE-based metrics offer a complementary and informative perspective that emphasizes conditional response behavior rather than variance attribution alone.}

     \item The value of $\sigma_{I_{\mathrm{ice}}}$ reflects only the overall interaction impact, not specifically on which variables are involved. Interacting variables can be identified using joint PDPs, two-way Sobol' indices, or two-way SHAP values. Combining these visualizations with GSA metrics ($\mu_{I_{\mathrm{ice}}}$, $\sigma_{I_{\mathrm{ice}}}$, and two-way GSA) enables a more comprehensive understanding of the model’s behavior.

    \item The ICE correlation values provide a quantitative view of how interactions modify input–output relationships, revealing changes in linearity or correlation. \textcolor{black}{Their key strength lies in capturing directional and structural heterogeneity in conditional response behavior, which is not accessible through variance-based sensitivity measures alone. }However, they should be interpreted with caution, as they indicate association rather than the magnitude of impact.
    
     \item PDP (with ICE) and SHAP visualizations generally convey similar insights, though PDP and ICE are often more intuitive as they depict smooth model slices and averages. Joint dependence and SHAP interaction plots, while distinct, offer complementary perspectives. Combined with the proposed GSA metrics, these visualizations enable a comprehensive understanding of the predictive model.

\end{enumerate}
Our findings highlight the utility of ICE-based feature importance. While $\mu_{I_{\mathrm{ice}}}$ offers a more comprehensive view of the input variable's impact by capturing effects that $I_{pdp}$ might miss, $\sigma_{I_{\mathrm{ice}}}$ further enriches the analysis by revealing how interactions with other variables influence this impact. \textcolor{black}{However, it is also important to discuss some limitations of the proposed metrics:
\begin{enumerate}
    \item All proposed metrics (\textit{i.e.}, $\mu_{I_{\mathrm{ice}}}$, $\sigma_{I_{\mathrm{ice}}}$ and $\sigma_{\rho}$) are not suitable for handling correlated input variables. This limitation stems from the use of ICE itself, which evaluates conditional responses by varying one input at a time while sampling others independently. While this does not pose an issue in applications with independent inputs, such as design exploration or parametric studies, care must be taken when applying the proposed metrics to problems involving correlated variables.
    \item The proposed $\sigma_{I_{\mathrm{ice}}}$, although useful to assess the impact of interactions of an input with other inputs, cannot distinguish which variables it strongly interacts with. Although visualization can help, it might be troublesome to look at every single other variable. It also lacks the capability to model higher-order interaction effects, such as third-order interactions. To that end, this analysis in terms of magnitude should be done together with the interaction Sobol' indices. 
    \item The current method focuses only on global sensitivity in terms of how the input affects the output; it cannot be applied to specific cases, such as reliability-based sensitivity analysis, which focuses on extreme cases. This is in contrast to Sobol' indices, which can be extended to reliability-based sensitivity analysis by using the indicator function in the calculation of the metrics~\cite{cui2010moment}.
\end{enumerate}
}

\textcolor{black}{The proposed metrics are model-agnostic and can be applied to any predictive model. When applied to a surrogate model constructed from data, however, the resulting sensitivity measures are evaluated with respect to the trained model and may therefore vary across different models approximating the same physical system. In such cases, the computed importance reflects the behavior learned by the surrogate model. Therefore, to accurately reflect the underlying modeled function, the surrogate itself must first achieve sufficient predictive accuracy.}

% Visualization techniques such as PDP with ICE curves and SHAP plots complement each other, with PDP and ICE offering more intuitive representations through smooth, easily interpretable curves. 

\section{Conclusion and future works}
The primary goal of this paper is to evaluate the effectiveness of PDP and ICE in supporting engineering design exploration and analysis. Most importantly, we introduce new metrics that enhance the insights obtainable from existing PDP- and ICE-based approaches. Specifically, we propose ICE-based feature importance metrics that address the alternating-association issue inherent in PDP-based importance. The key idea is to compute the mean and standard deviation of the feature importance derived from each ICE curve, providing both a global average effect and information on how interactions with other variables influence a feature’s significance. This allows users to assess not only the average impact of an input variable but also its potential strongest or weakest influence on the predictive model. The standard deviation summarizes the effect of higher-order interactions in a single value. Additionally, we assess the correlation between ICE curves and PDP to further examine how interactions modify the input–output relationship.

We applied PDP and ICE to a 5-variable Friedman function, a 5-variable wind-turbine fatigue problem, and a 9-variable airfoil dataset to demonstrate the insights enabled by these methods and the new sensitivity metrics. PDP and ICE reveal information not easily obtained from SHAP, and each method contributes a distinct perspective. The proposed metrics further provide quantitative interaction-aware sensitivity measures that strengthen knowledge discovery in engineering design analysis.

\textcolor{black}{Despite these contributions, several limitations warrant discussion, and addressing them forms a central topic for future research. First, all proposed metrics ($\mu_{I_{\mathrm{ice}}}$, $\sigma_{I_{\mathrm{ice}}}$, and $\sigma_{\rho}$) inherit the limitations of ICE and are not suitable for correlated input variables. Second, while $\sigma_{I_{\mathrm{ice}}}$ summarizes the magnitude of interaction effects, it cannot identify which specific variables are responsible for strong interactions; thus, it should be complemented with interaction Sobol' indices when detailed decomposition is needed. Third, the framework focuses on global sensitivity in terms of average input--output influence and is not directly applicable to reliability-based sensitivity analysis that emphasizes extreme events.}

Future work may also explore integrating PDP and ICE insights into optimization frameworks, using the proposed sensitivity metrics to guide design decisions by prioritizing influential variables. Advanced visualization techniques that dynamically reveal interactions captured by ICE curves and PDPs would further improve interpretability, especially in complex, nonlinear systems. \textcolor{black}{Incorporating uncertainty into the proposed metrics, especially uncertainty arising from surrogate-model approximation errors, necessitates a complete uncertainty-quantification pipeline. %(e.g., via bootstrapping). 
For probabilistic surrogate models, such as Gaussian process regression, predictive uncertainty information could be explicitly incorporated to quantify confidence in the extracted sensitivity measures. Developing such an integrated framework would provide a more comprehensive characterization of uncertainty in high-fidelity engineering simulations. Finally, automating insight extraction from PDPs, ICE curves, and the proposed metrics could reduce reliance on expert judgment and improve usability in real-world engineering workflows.
}

\section*{Acknowledgements}
\noindent Pramudita Satria Palar would like to acknowledge financial support from Institut Teknologi Bandung through the Riset Kolaborasi Internasional 2024 scheme. This research was also funded by the Indonesian Endowment Fund for Education (LPDP) on behalf of the Indonesian Ministry of Higher Education, Science and Technology and managed under the EQUITY Program (Contract No. 8128/IT1.B07.1/TA.00/2025). This work is part of the activities of ONERA - ISAE - ENAC joint research group. The research presented in this paper has been performed in the framework of the COLOSSUS project (Collaborative System of Systems Exploration of Aviation Products, Services and Business Models) and has received funding from the European Union Horizon Europe program under grant agreement n${^\circ}$ 101097120 and in the MIMICO research project funded by the Agence Nationale de la Recherche (ANR) n$^o$ ANR-24-CE23-0380. 
\section*{Declaration of  Competing interest}
 \noindent Pramudita Satria Palar reports that financial support was provided by Bandung Institute of Technology and LPDP. Nicolas Verstaevel and Paul Saves report that financial support was provided by the French National Research Agency. Joseph Morlier reports that financial support was provided by ISAE-SUPAERO Higher Institute of Aerospace Engineering. If there are other authors, they declare that they have no known competing financial interests or personal relationships that could have appeared to influence the work reported in this paper.
 
\section*{Data availability and replication}
\noindent The data on which this work is based is available on GitHub: \url{https://github.com/kanakaero/Dataset-of-Aerodynamic-and-Geometric-Coefficients-of-Airfoils} (the data do not belong to the authors).

\noindent Both the produced data and the code developed to reproduce the results on aerodynamics and airfoil problems will be available upon request.

\appendix
% \newpage
\section{Proving Inequalities Between PDP and ICE-based global sensitivity metrics}In this section, we prove an inequality involving the PD Function and the ICE for a large class of functions. Below, let $1 \le j \le m$ and let $C = \{1,2,\ldots,m\} \setminus \{j\}$. We have an interest in the GSA metrics for a single variable indexed by $j$ (\textit{i.e.}, $x_j$), with $\boldsymbol{x}_C$ referring to other variables indexed in $C$. Also, let $\mathbb{E}[\cdot]$ and $\mathbb{V}[\cdot]$ denote the expectation and the variance operator, respectively, and the subscript denotes that the operation is being taken with respect to a specific variable. In the following explanation, we use the notation for a general function $f(\boldsymbol{x})$ to derive the proof. We wish to verify the following inequality for a class of functions $f(x_j, \boldsymbol{x}_C)$:
\begin{eqnarray*}
\mathbb{E}_{\boldsymbol{x}_C}\left[ I_{\mathrm{ice}}(x_j; \boldsymbol{x}_C) \right] \ge I_{pdp}(x_j),
\end{eqnarray*}
or equivalently,
\begin{eqnarray*}
\mathbb{E}_{\boldsymbol{x}_C}\left[ \sqrt{\mathbb{V}_{x_j}[f(x_j, \boldsymbol{x}_C)]} \right] \ge \sqrt{ \mathbb{V}_{x_j} \left[ \mathbb{E}_{\boldsymbol{x}_C}[f(x_j, \boldsymbol{x}_C)] \right] }.
\end{eqnarray*}

\begin{theorem}
If $f(x_j, \boldsymbol{x}_C) = g(x_j) + h(\boldsymbol{x}_C)$, where $g$ is a function of $x_j$ only and $h$ is a function of $\boldsymbol{x}_C$ only, then $\mathbb{E}_{\boldsymbol{x}_C}\left[ I_{\mathrm{ice}}(x_j; \boldsymbol{x}_C) \right] = I_{pdp}(x_j)$.
\end{theorem}
\vspace{3mm}

\begin{proof} 
Let $f(x_j, \boldsymbol{x}_C) = g(x_j)+h(\boldsymbol{x}_C)$ where $g$ is a function of $x_j$ only and $h$ is a function of $\boldsymbol{x}_C$ only. First, note that
\begin{eqnarray*}
\mathbb{E}_{\boldsymbol{x}_C}\left[ \sqrt{\mathbb{V}_{x_j}[f(x_j, \boldsymbol{x}_C)]} \right] = \mathbb{E}_{\boldsymbol{x}_C}\left[ \sqrt{\mathbb{V}_{x_j}[g(x_j)+h(\boldsymbol{x}_C)]} \right] = \mathbb{E}_{\boldsymbol{x}_C}\left[ \sqrt{\mathbb{V}_{x_j}[g(x_j)]} \right] =
\sqrt{\mathbb{V}_{x_j}[g(x_j)]}.
\end{eqnarray*}
Next, note that
\begin{eqnarray*}
\sqrt{ \mathbb{V}_{x_j} \left[ \mathbb{E}_{\boldsymbol{x}_C}[f(x_j, \boldsymbol{x}_C)] \right] } = \sqrt{ \mathbb{V}_{x_j} \left[ \mathbb{E}_{\boldsymbol{x}_C}[g(x_j) + h(\boldsymbol{x}_C)] \right] } = \sqrt{ \mathbb{V}_{x_j} \left[ g(x_j) + \mathbb{E}_{\boldsymbol{x}_C}[h(\boldsymbol{x}_C)] \right] } = \sqrt{\mathbb{V}_{x_j}[g(x_j)]}.
\end{eqnarray*}
Hence,
\begin{eqnarray*}
\mathbb{E}_{\boldsymbol{x}_C}\left[ \sqrt{\mathbb{V}_{x_j}[f(x_j, \boldsymbol{x}_C)]} \right] = \sqrt{ \mathbb{V}_{x_j} \left[ \mathbb{E}_{\boldsymbol{x}_C}[f(x_j, \boldsymbol{x}_C)] \right] }.
\end{eqnarray*}
\end{proof}
\vspace{3mm}

\begin{theorem}
If $f(x_j, \boldsymbol{x}_C) = g(x_j)h(\boldsymbol{x}_C)$, where $g$ is a function of $x_j$ only and $h$ is a function of $\boldsymbol{x}_C$ only, then $\mathbb{E}_{\boldsymbol{x}_C}\left[ I_{\mathrm{ice}}(x_j; \boldsymbol{x}_C) \right] \ge I_{pdp}(x_j)$. Moreover, equality holds if $h$ is a nonnegative function. 
\end{theorem}
\vspace{3mm}

\begin{proof}
Let $f(x_j, \boldsymbol{x}_C) = g(x_j)h(\boldsymbol{x}_C)$ where $g$ is a function of $x_j$ only and $h$ is a function of $\boldsymbol{x}_C$ only. First, note that
\begin{eqnarray*}
\begin{array}{rcl}
\mathbb{E}_{\boldsymbol{x}_C}\left[ \sqrt{\mathbb{V}_{x_j}[f(x_j, \boldsymbol{x}_C)]} \right] & = & \mathbb{E}_{\boldsymbol{x}_C}\left[ \sqrt{\mathbb{V}_{x_j}[g(x_j)h(\boldsymbol{x}_C)]} \right] =  \mathbb{E}_{\boldsymbol{x}_C}\left[ \sqrt{h(\boldsymbol{x}_C)^2 \mathbb{V}_{x_j}[g(x_j)]} \right] \\
\\
& = & \mathbb{E}_{\boldsymbol{x}_C}\left[ |h(\boldsymbol{x}_C)| \sqrt{\mathbb{V}_{x_j}[g(x_j)]} \right] = \sqrt{\mathbb{V}_{x_j}[g(x_j)]} \mathbb{E}_{\boldsymbol{x}_C}[|h(\boldsymbol{x}_C)|]. \\
\end{array}
\end{eqnarray*}
Next, we have
\begin{eqnarray*}
\begin{array}{rcl}
\sqrt{ \mathbb{V}_{x_j} \left[ \mathbb{E}_{\boldsymbol{x}_C}[f(x_j, \boldsymbol{x}_C)] \right] } & = & \sqrt{ \mathbb{V}_{x_j} \left[ \mathbb{E}_{\boldsymbol{x}_C}[g(x_j)h(\boldsymbol{x}_C)] \right] } = \sqrt{ \mathbb{V}_{x_j} \left[ g(x_j) \mathbb{E}_{\boldsymbol{x}_C}[h(\boldsymbol{x}_C)] \right] } \\
\\
& = & \sqrt{ (\mathbb{E}_{\boldsymbol{x}_C}[h(\boldsymbol{x}_C)])^2   \mathbb{V}_{x_j}[g(x_j)] } = |\mathbb{E}_{\boldsymbol{x}_C}[h(\boldsymbol{x}_C)]| \sqrt{\mathbb{V}_{x_j}[g(x_j)]}.
\end{array}
\end{eqnarray*}
By Jensen's inequality, $\mathbb{E}_{\boldsymbol{x}_C}[|h(\boldsymbol{x}_C)|] \ge |\mathbb{E}_{\boldsymbol{x}_C}[h(\boldsymbol{x}_C)]|$, and so,
\begin{eqnarray*}
\begin{array}{rcl}
\mathbb{E}_{\boldsymbol{x}_C}\left[ \sqrt{\mathbb{V}_{x_j}[f(x_j, \boldsymbol{x}_C)]} \right] = \sqrt{\mathbb{V}_{x_j}[g(x_j)]} \mathbb{E}_{\boldsymbol{x}_C}[|h(\boldsymbol{x}_C)|] & \ge & |\mathbb{E}_{\boldsymbol{x}_C}[h(\boldsymbol{x}_C)]| \sqrt{\mathbb{V}_{x_j}[g(x_j)]} \\
\\
& = & \sqrt{ \mathbb{V}_{x_j} \left[ \mathbb{E}_{\boldsymbol{x}_C}[f(x_j, \boldsymbol{x}_C)] \right] } \\
\end{array}
\end{eqnarray*}
Moreover, if $h$ is a nonnegative function, then $\mathbb{E}_{\boldsymbol{x}_C}[h(\boldsymbol{x}_C)] \ge 0$. In this case, $\mathbb{E}_{\boldsymbol{x}_C}[|h(\boldsymbol{x}_C)|] = |\mathbb{E}_{\boldsymbol{x}_C}[h(\boldsymbol{x}_C)]|$ and equality holds above.
\end{proof}
\vspace{3mm}

Before we state and prove the next theorem that covers a more general class of functions for which the above inequality holds, we first prove the following lemma.
\vspace{3mm}

\begin{lemma}
If $A \in \mathbb{R}^{k \times k}$ is a symmetric and positive semidefinite matrix, then $\varphi(z) = \sqrt{ z^T A z },\ z \in \mathbb{R}^k$ is a convex function. Moreover, if $Z$ is a $k$-dimensional random vector, then $E[\sqrt{ Z^T A Z }] \ge \sqrt{ E(Z)^T A E(Z) }$. 
\end{lemma}
\vspace{3mm}

\begin{proof}
Suppose $A$ is a symmetric positive semidefinite matrix. Then, $A = M^T M$ for some matrix $M \in \mathbb{R}^{k \times k}$ (\textit{e.g.}, $M$ may be obtained via spectral decomposition). Hence,
\begin{eqnarray*}
\varphi(z) = \sqrt{ z^T A z} = \sqrt{ z^T M^T M z } = \sqrt { (Mz)^T (Mz) } = \|Mz\|,
\end{eqnarray*}
where $\| \cdot \|$ is the 2-norm in $\mathbb{R}^k$. To prove the convexity of $\varphi(z)$, let $z_1, z_2 \in \mathbb{R}^k$ and let $0 \le \lambda \le 1$. Note that
\begin{eqnarray*}
\begin{array}{rcl}
\varphi(\lambda z_1 + (1-\lambda)z_2) = \|M(\lambda z_1 + (1-\lambda)z_2)\| & = & \| \lambda M z_1 + (1-\lambda) M z_2\| \\
& \le & \lambda \|Mz_1\| + (1-\lambda) \|Mz_2\| \\
& = & \lambda \varphi(z_1) + (1-\lambda) \varphi(z_2), \\
\end{array}
\end{eqnarray*}
where the above inequality follows from the triangle inequality. This shows that
\begin{eqnarray*}
\varphi(\lambda z_1 + (1-\lambda)z_2) \le \lambda \varphi(z_1) + (1-\lambda) \varphi(z_2),
\end{eqnarray*}
and so, $\varphi(z)$ is a convex function. Now, by Jensen's Inequality, we have $E[\varphi(Z)] \ge \varphi(E[Z])$, or equivalently, $E[\sqrt{ Z^T A Z }] \ge \sqrt{ E(Z)^T A E(Z) }$.
\end{proof}
\vspace{3mm}

\begin{theorem}
If $f(x_j, \boldsymbol{x}_C) = g_1(x_j)h_1(\boldsymbol{x}_C) + \ldots + g_k(x_j)h_k(\boldsymbol{x}_C)$, where $g_1,\ldots,g_k$ are functions of $x_j$ only and $h_1,\ldots,h_k$ are functions of $\boldsymbol{x}_C$ only, then $\mathbb{E}_{\boldsymbol{x}_C}\left[ I_{\mathrm{ice}}(x_j; \boldsymbol{x}_C) \right] \ge I_{pdp}(x_j)$.
\end{theorem}
\vspace{3mm}

\begin{proof}
Let $f(x_j, \boldsymbol{x}_C) = g_1(x_j)h_1(\boldsymbol{x}_C) + \ldots + g_k(x_j)h_k(\boldsymbol{x}_C)$, where $g_1,\ldots,g_k$ are functions of $x_j$ only and $h_1,\ldots,h_k$ are functions of $\boldsymbol{x}_C$ only. First, note that
\begin{eqnarray*}
\begin{array}{rcl}
\mathbb{V}_{x_j}[f(x_j, \boldsymbol{x}_C)] & = & \mathbb{V}_{x_j}[ g_1(x_j) h_1(\boldsymbol{x}_C) + \ldots + g_k(x_j) h_k(\boldsymbol{x}_C) ] \\
\\
& = & \displaystyle{ \sum_{\ell=1}^k h_{\ell}(\boldsymbol{x}_C)^2 \mathbb{V}_{x_j}[g_{\ell}(x_j)] + 2 \sum_{1 \le \ell < r \le k} h_{\ell}(\boldsymbol{x}_C) h_r(\boldsymbol{x}_C) \mbox{Cov}_{x_j}(g_{\ell}(x_j), g_r(x_j))} \\
\end{array}
\end{eqnarray*}
Define the random vectors $G(x_j) = [g_1(x_j), \ldots, g_k(x_j)]^T$ and $H(\boldsymbol{x}_C) = [h_1(\boldsymbol{x}_C), \ldots, h_k(\boldsymbol{x}_C)]^T$. Now, note that 
\begin{eqnarray*}
\mathbb{V}_{x_j}[f(x_j, \boldsymbol{x}_C)] =  \mathbb{V}_{x_j}[ H(\boldsymbol{x}_C)^T G(x_j) ] = H(\boldsymbol{x}_C)^T \mbox{Cov}_{x_j}(G(x_j)) H(\boldsymbol{x}_C),
\end{eqnarray*}
where
\begin{eqnarray*}
\mbox{Cov}_{x_j}(G(x_j)) =
\begin{bmatrix}
\mathbb{V}_{x_j}[g_1(x_j)] & \mbox{Cov}_{x_j}(g_1(x_j), g_2(x_j)) & \ldots & \mbox{Cov}_{x_j}(g_1(x_j), g_k(x_j)) \\
\mbox{Cov}_{x_j}(g_2(x_j), g_1(x_j)) & \mathbb{V}_{x_j}[g_2(x_j)] & \ldots & \mbox{Cov}_{x_j}(g_2(x_j), g_k(x_j)) \\
\vdots & \vdots & \ddots & \vdots \\
\mbox{Cov}_{x_j}(g_k(x_j), g_1(x_j)) & \mbox{Cov}_{x_j}(g_k(x_j), g_2(x_j)) & \ldots & \mathbb{V}_{x_j}[g_k(x_j)] \\
\end{bmatrix}
\end{eqnarray*}
is the covariance matrix of $G(x_j)$. Hence,
\begin{eqnarray*}
\mathbb{E}_{\boldsymbol{x}_C}\left[ I_{\mathrm{ice}}(x_j; \boldsymbol{x}_C) \right] = \mathbb{E}_{\boldsymbol{x}_C}\left[ \sqrt{\mathbb{V}_{x_j}[f(x_j, \boldsymbol{x}_C)]} \right] = \mathbb{E}_{\boldsymbol{x}_C}\left[ \sqrt{ H(\boldsymbol{x}_C)^T \mbox{Cov}_{x_j}(G(x_j)) H(\boldsymbol{x}_C) } \right].
\end{eqnarray*}
Next, note that
\begin{eqnarray*}
\begin{array}{c}
\mathbb{V}_{x_j} \left[ \mathbb{E}_{\boldsymbol{x}_C}[f(x_j, \boldsymbol{x}_C)] \right] = \mathbb{V}_{x_j} \left[ \mathbb{E}_{\boldsymbol{x}_C}[ g_1(x_j) h_1(\boldsymbol{x}_C) + \ldots + g_k(x_j) h_k(\boldsymbol{x}_C) ] \right] \\
\\
= \mathbb{V}_{x_j} \left[ g_1(x_j) \mathbb{E}_{\boldsymbol{x}_C}[h_1(\boldsymbol{x}_C)] + \ldots + g_k(x_j) \mathbb{E}_{\boldsymbol{x}_C}[h_k(\boldsymbol{x}_C)] \right] \\
\\
= \displaystyle{ \sum_{\ell=1}^k (\mathbb{E}_{\boldsymbol{x}_C}[h_{\ell}(\boldsymbol{x}_C)])^2 \mathbb{V}_{x_j}[g_{\ell}(x_j)] + 2 \sum_{1 \le \ell < r \le k} \mathbb{E}_{\boldsymbol{x}_C}[h_{\ell}(\boldsymbol{x}_C)] \mathbb{E}_{\boldsymbol{x}_C}[h_r(\boldsymbol{x}_C)] \mbox{Cov}_{x_j}(g_{\ell}(x_j), g_r(x_j))} \\
\end{array}
\end{eqnarray*}
Since $\mathbb{E}_{\boldsymbol{x}_C}[H(\boldsymbol{x}_C)] = [\mathbb{E}_{\boldsymbol{x}_C}[h_1(\boldsymbol{x}_C)], \ldots, \mathbb{E}_{\boldsymbol{x}_C}[h_k(\boldsymbol{x}_C)]]^T$, we get
\begin{eqnarray*}
\mathbb{V}_{x_j} \left[ \mathbb{E}_{\boldsymbol{x}_C}[f(x_j, \boldsymbol{x}_C)] \right] = \mathbb{E}_{\boldsymbol{x}_C}[H(\boldsymbol{x}_C)]^T \mbox{Cov}_{x_j}(G(x_j)) \mathbb{E}_{\boldsymbol{x}_C}[H(\boldsymbol{x}_C)],
\end{eqnarray*}
and so
\begin{eqnarray*}
I_{pdp}(x_j) = \sqrt{ \mathbb{V}_{x_j} \left[ \mathbb{E}_{\boldsymbol{x}_C}[f(x_j, \boldsymbol{x}_C)] \right] } = \sqrt{ \mathbb{E}_{\boldsymbol{x}_C}[H(\boldsymbol{x}_C)]^T \mbox{Cov}_{x_j}(G(x_j)) \mathbb{E}_{\boldsymbol{x}_C}[H(\boldsymbol{x}_C)] }.
\end{eqnarray*}

Now, using the previous lemma with $Z = H(\boldsymbol{x}_C)$ and with $A=\mbox{Cov}_{x_j}(G(x_j))$, we obtain
\begin{eqnarray*}
\begin{array}{rcl} 
\mathbb{E}_{\boldsymbol{x}_C}\left[ I_{\mathrm{ice}}(x_j; \boldsymbol{x}_C) \right] & = & \mathbb{E}_{\boldsymbol{x}_C}\left[ \sqrt{ H(\boldsymbol{x}_C)^T \mbox{Cov}_{x_j}(G(x_j)) H(\boldsymbol{x}_C) } \right] \\
\\
& \ge & \sqrt{ \mathbb{E}_{\boldsymbol{x}_C}[H(\boldsymbol{x}_C)]^T \mbox{Cov}_{x_j}(G(x_j)) \mathbb{E}_{\boldsymbol{x}_C}[H(\boldsymbol{x}_C)] } = I_{pdp}(x_j). \\
\end{array}
\end{eqnarray*}
\end{proof}
\vspace{3mm}

\begin{corollary}
If $f$ is a polynomial in $x_1,\ldots,x_m$, then $\mathbb{E}_{\boldsymbol{x}_C}\left[ I_{\mathrm{ice}}(x_j; \boldsymbol{x}_C) \right] \ge I_{pdp}(x_j)$ for any $j=1,\ldots,m$.
\end{corollary}

\begin{proof}
If $f$ is a polynomial in $x_1,\ldots,x_m$, then for any $j=1,\ldots,m$, $f$ has the form $f(x_j, \boldsymbol{x}_C) = g_1(x_j)h_1(\boldsymbol{x}_C) + \ldots + g_k(x_j)h_k(\boldsymbol{x}_C)$, where $g_1,\ldots,g_k$ are functions of $x_j$ only and $h_1,\ldots,h_k$ are functions of $\boldsymbol{x}_C$ only. The result follows from the previous theorem.
\end{proof}
\vspace{3mm}

The above corollary shows that the ICE-PDP inequality holds for all multivariate polynomials. However, it holds for a much larger class of functions that has some combination of additive and multiplicative separability, such as
\begin{eqnarray*}
f(x_1,x_2) = x_1^2 \sin(\pi x_2) + e^{x_1} {x_2}^3 + \cos(\pi x_1) \sqrt{x_2}.
\end{eqnarray*}
\textcolor{black}{Various surrogate and machine-learning models admit such a structure, either exactly or approximately. Examples include support vector regression models with polynomial kernels, as well as expressions obtained via genetic programming when the search space is constrained to enforce additive or multiplicative separability. However, for more complex models, such as feedforward neural networks with multiple nonlinear hidden layers, the above proof is not directly applicable due to the absence of explicit additive or multiplicative separability. We conjecture that the inequality also holds for more general functions, though an investigation of this claim will be left for future work.}

% The orthogonal polynomial expansion with finite truncation in PCE (see Sec.~\ref{sec:surrogatemodel}) adheres to this expansion. Consequently, the inequality holds for the truncated PCE.

\bibliographystyle{elsarticle-num}
\bibliography{sample_origin.bib}
\end{document}